\documentclass[letterpaper]{article} 
\usepackage{aaai2026}  
\usepackage{times}  
\usepackage{helvet}  
\usepackage{courier}  
\usepackage[hyphens]{url}  
\usepackage{graphicx} 
\usepackage{natbib}  
\usepackage{caption} 
\usepackage{algorithm}

\usepackage{newfloat}
\usepackage{listings}
\DeclareCaptionStyle{ruled}{labelfont=normalfont,labelsep=colon,strut=off} 
\floatstyle{ruled}
\newfloat{listing}{tb}{lst}{}
\floatname{listing}{Listing}

\nocopyright

\usepackage{amsfonts}
\usepackage{makecell}
\usepackage{mathtools}
\usepackage{amsthm}
\usepackage{thmtools}
\usepackage[noend]{algpseudocode}
\usepackage{xspace}
\usepackage{siunitx}
\usepackage{enumitem} 
\usepackage{tabularx}

\newtheorem{theorem}{Theorem}
\newtheorem{lemma}[theorem]{Lemma}
\newtheorem{corollary}[theorem]{Corollary}
\newtheorem{proposition}[theorem]{Proposition}

\newtheorem{definition}{Definition}

\newcommand{\model}{\textit{Model}}
\newcommand{\vc}{\textit{VC}}
\newcommand{\algname}{\textsc{Woodelf}\xspace} 
\newcommand{\shap}{\textsc{shap}\xspace}

\newcommand{\bo}{\beta^{\textit{original}}_i}
\newcommand{\bs}{\beta^{\textit{simplified}}_i}
\newcommand{\so}{\phi^{\textit{original}}_i}
\newcommand{\shapsimp}{\phi^{\textit{simplified}}_i}

\title{From Decision Trees to Boolean Logic: A Fast and Unified SHAP Algorithm}
\author {
    Alexander Nadel\textsuperscript{\rm 1},
    Ron Wettenstein\textsuperscript{\rm 2}
}

\affiliations {
    \textsuperscript{\rm 1}Faculty of Data and Decision Sciences, Technion, Haifa, Israel, alexandernad@technion.ac.il\\
    \textsuperscript{\rm 2}Reichman University, Herzliya, Israel, ron.wettenstein@post.runi.ac.il
}
\date{April 2025}

\begin{document}

\maketitle

\begin{abstract}

SHapley Additive exPlanations (SHAP) is a key tool for interpreting decision tree ensembles by assigning contribution values to features. It is widely used in finance, advertising, medicine, and other domains. Two main approaches to SHAP calculation exist: \textit{Path-Dependent SHAP}, which leverages the tree structure for efficiency, and \textit{Background SHAP}, which uses a background dataset to estimate feature distributions.

We introduce \algname, a SHAP algorithm that integrates decision trees, game theory, and Boolean logic into a unified framework. 
For each consumer, \algname constructs a pseudo-Boolean formula that captures their feature values, the structure of the decision tree ensemble, and the entire background dataset. It
then leverages this representation to compute Background SHAP in linear time. \algname can also compute Path-Dependent SHAP, Shapley interaction values, Banzhaf values, and Banzhaf interaction values.

\algname is designed to run efficiently on CPU and GPU hardware alike. Available via the \algname Python package, it is implemented using NumPy, SciPy, and CuPy without relying on custom C++ or CUDA code. This design enables fast performance and seamless integration into existing frameworks, supporting large-scale computation of SHAP and other game-theoretic values in practice.

For example, on a dataset with \num{3000000} rows, \num{5000000} background samples, and \num{127} features, \algname computed all Background Shapley values in \num{162} seconds on CPU and \num{16} seconds on GPU—compared to 44 minutes required by the best method on any hardware platform, representing 16$\times$ and 165$\times$ speedups, respectively.\\[2pt]

\textit{This is the full version of our paper to appear in AAAI-26.}

\end{abstract}

\begin{links}
    \link{The \algname Python Package}{https://github.com/ron-wettenstein/woodelf}
    \link{Experiment Notebooks}{https://github.com/ron-wettenstein/WoodelfExperiments}
    \link{IEEE-CIS Dataset}{https://www.kaggle.com/c/ieee-fraud-detection}
    \link{KDD-Cup Dataset}{https://kdd.ics.uci.edu/databases/kddcup99/kddcup99.html}
\end{links}

\section{Introduction}

Decision trees are widely used predictive models for classification and regression. 
To improve predictive accuracy, ensemble methods such as XGBoost~\cite{xgboost_paper}, Random Forest~\cite{random_forrest}, 
and CatBoost~\cite{catboost}, train multiple decision trees and average their predictions.

Recent efforts focus on explaining models using feature attributions. Local attribution shows how each feature affects a single prediction, which is often crucial for regulatory compliance~\cite{adverse_action_ecoa, onGDPRregulation}. Global attribution evaluates which features matter most overall, often by combining many local attributions~\cite{understandingglobalfeaturecontributions}. This is essential for model comprehension and feature selection.

SHAP (SHapley Additive exPlanations)~\cite{explainable_ai_trees} is a widely preferred method for both local and global feature attribution~\cite{responsible_ml, ml_interpretability_book}. It assigns Shapley values to features, providing a unified~\cite{unified_approach_to_interpreting_models} and consistent~\cite{consistent_shap} approach grounded in game theory.

\subsection{Shapley Values}
\label{sec:shap_intro}
Originating from cooperative game theory, Shapley values offer a fair method for distributing profits among players based on their individual contributions. 
Players whose contributions are crucial to the group's success receive a larger share of the profit, while those with smaller contributions receive less. Players who negatively impact the group's performance may receive negative payments.

Shapley values constitute the unique solution satisfying four key properties: efficiency, null player, symmetry, and linearity~\cite{original_shapley_paper}. The Shapley value formula uses the game's characteristic function to evaluate the player's impact across all possible coalitions.

\begin{definition}[Characteristic Function]
\label{definition_characteristic_function} For a group of players $N$, the \emph{characteristic function} $M$ is a function of the form $M:2^{[N]}\to\mathbb{R}$ mapping any subset $S\subseteq N$ to the profit when only players in $S$ participate.
\end{definition}

The Shapley value formula for player $i$, shown below, considers all subsets excluding $i$ and compares the profit with and without the player. These contributions are then normalized by a factor that depends on the coalition size.

\begin{equation}
\label{eq:s}
\phi_i(M) = \sum\limits_{S \subseteq N\setminus \{i\}} \!\!\!\!\!\frac{|S|!(|N| - |S| - 1)!}{|N|!}(M(S\cup\{i\}) - M(S))
\end{equation}

A naive Shapley values calculation takes exponential time, as the formula considers all possible subsets of $N\setminus\{i\}$. 
However, efficient computation is possible for certain scenarios, including decision tree ensembles \cite{explainable_ai_trees}, certain Boolean circuits classes \cite{arenas2021tractabilityshapscorebasedexplanationsdeterministic}, and tuples in query answering \cite{shapley_on_queries}. For other scenarios like neural networks where the computation is \#P-Hard~\cite{tractability_shap_explanations, complexity_of_SHAP}, approximation methods exist \cite{shapley_on_ML}.

\subsection{Feature Importance Using Shapley Values}
\label{sec:feature_importance_using_shap}
A predictive model, like a decision tree ensemble, can be viewed as a ``game'', where feature values act as ``players'' and the model's prediction represents the ``game's profit''. In this context, Shapley values quantify how each feature (''player'') affected the prediction (the ''game's profit''). 

Our goal is to define the characteristic function of this “game” and compute its Shapley values. This requires specifying the model’s output when only a subset of features is present, while the rest are considered missing. Several definitions exist for handling missing features; see~\cite{many_missing_feature_definitions} for a survey. Of these, three are most commonly used today:

\begin{itemize}
    \item \emph{Baseline SHAP}: Assigns each feature a fixed baseline value used whenever the feature is missing.

    \item \emph{Path-Dependent SHAP}: Leverages the tree structure and the cover property (i.e., how many rows reached each node during training) to infer the effect of missing features.

    \item \emph{Background SHAP}: Replaces fixed baselines with a background dataset. When features are missing, their values are taken from this dataset, predictions are computed, and results are averaged. This method is the most accurate, see Appendix~\ref{sec:background_shap_is_the_most_accurate} for an example.
\end{itemize}

\cite{explainable_ai_trees} were the first to present a polynomial-time algorithm for the three SHAP approaches discussed above. Their method efficiently tracks the number of feature subsets that reach each node, avoiding the need to enumerate them explicitly. These algorithms are implemented in the widely used \shap Python package.

Since then, several works have improved these methods. FastTreeShap v2~\cite{fast_tree_shap} accelerates Path-Dependent SHAP by extracting key information from the decision tree in advance. PLTreeShap~\cite{linear_background_shap}, which preprocesses both the decision tree and the background dataset, reduces the complexity of Background SHAP from $O(mn)$—where $n$ is the number of consumers and $m$ is the background size—to $O(m + n)$.

GPU-based implementations include GPUTreeSHAP~\cite{GPUTreeShap}, which computes Shapley values for each path in parallel and uses optimized bin-packing techniques to distribute work across GPU warps, and FourierSHAP~\cite{shap_fourier_gpu}, which performs well on models with a small number of features.

\subsection{Paper Outline and Our Contribution}
Sects.~\ref{sec:ltsc} and~\ref{sec:more_than_shap} introduce a novel linear-time algorithm for computing Shapley values, Shapley interaction values, Banzhaf values, and Banzhaf interaction values~\cite{PB_shap_and_banzhaf} for formulas in Weighted Disjunctive Normal Form (WDNF)~\cite{DBLP:conf/dna/ZhangJ05}.
Our main contribution, \algname, is introduced in Sect.~\ref{sec:woodelf_alg}, with several preliminary steps presented in Sects.~\ref{sec:decision_pattern},~\ref{sec:baseline_shap},~\ref{sec:baseline_shap_impl},~\ref{sec:bg_eq} and~\ref{sec:pd_eq}.

\algname is a \textbf{unified}, \textbf{generic}, \textbf{GPU-friendly} and \textbf{efficient} approach for SHAP calculation:  
\begin{itemize}
    \item \textbf{Unified} across Path-Dependent and Background SHAP, demonstrating that a single algorithm can handle the two problems previously thought to require distinct approaches.

\item \textbf{Metric-Generic}: \algname can compute Shapley values, Shapley interaction values, Banzhaf values, Banzhaf interaction values, and any other value over the Path-Dependent or Background characteristic functions that satisfies the linearity property. It is the first algorithm that supports such a broad range of metrics.

\item \textbf{GPU-friendly} and \textbf{pure Python}: Unlike existing approaches, in \algname, all the major algorithmic steps can be expressed as standard vectorized operations, making the algorithm inherently Single Instruction, Multiple Data (SIMD)- and GPU-friendly. Implementation-wise \algname is written entirely in Python, with the bottleneck operations implemented in NumPy and SciPy. By using CuPy, these operations can seamlessly run on GPUs without any custom \textsc{cuda} code. In contrast, \shap and other state-of-the-art (SOTA) implementations rely heavily on custom C++ and \textsc{cuda}. \algname achieves high efficiency while maintaining a pure Python design, simplifying integration and extensibility.

\item \textbf{Efficient}: By utilizing vectorized operations alongside efficient algorithmics (Sect.~\ref{sec:speed_ups}), \algname significantly advances SOTA performance. In Sect.~\ref{sec:experimental_results}, we demonstrate \algname's effectiveness on two large industrial datasets, achieving 24$\times$ to 333$\times$ speed-ups on GPU and 16$\times$ to 31$\times$ speed-ups on CPU, compared to the SOTA Background SHAP on any hardware platform.

\end{itemize}

\section{Linear-Time SHAP Calculation for WDNF}
\label{sec:ltsc}
Towards defining WDNF, we need additional notations. A \emph{literal} is a Boolean variable $x_i$ or its negation $\neg x_i$. A \emph{cube} is a conjunction (set) of literals.

\begin{figure*}[t]
\centering
\includegraphics[width=0.98\textwidth]{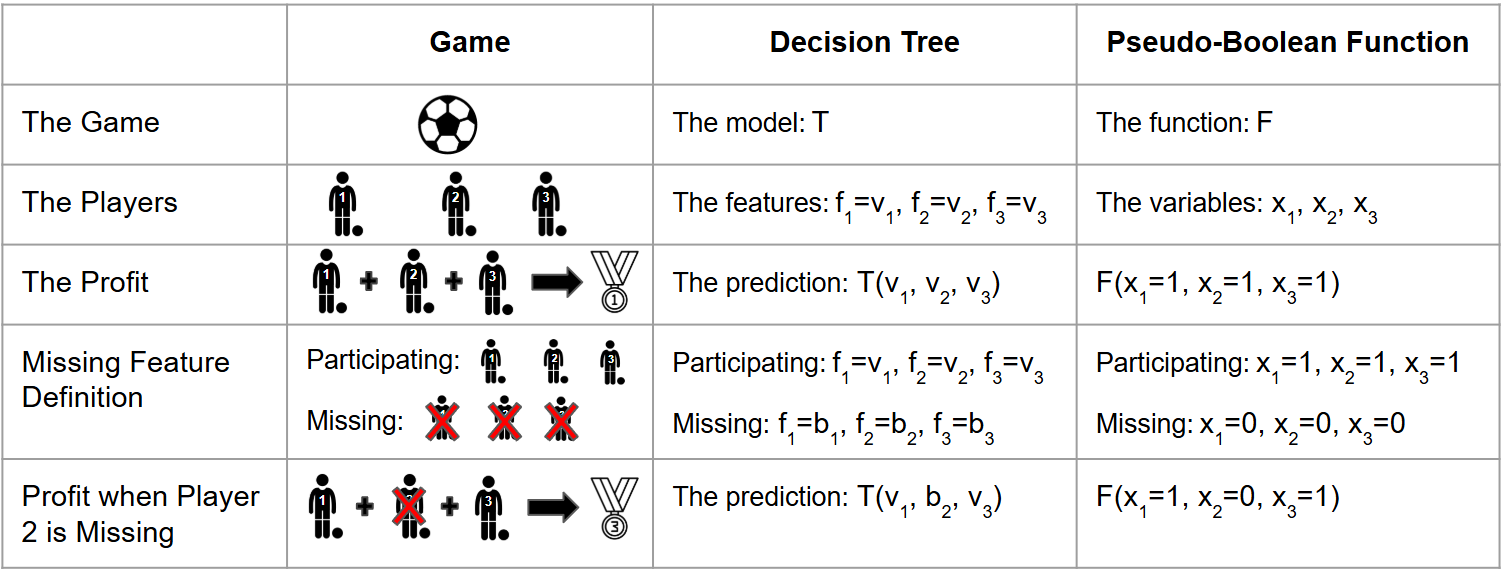} 
\caption{
An illustration how both PB functions and decision trees relate to well-established concepts in game theory. In decision trees, the model represents a game, features serve as players, and the prediction corresponds to profit. Under the baseline characteristic function definition, a missing player (e.g., player 2) is set to its baseline value ($b_2$) before making a prediction. In PB functions, the function itself represents a game, and the variables serve as players. Each variable is \texttt{True} when it participates and \texttt{False} when it is missing.}
\label{fig:CF_defs}
\end{figure*}

\begin{definition}[Pseudo-Boolean (PB) Function and Weighted Disjunctive Normal  Form (WDNF)] \label{def:WDNF}

A \emph{pseudo-Boolean (PB)} function is a function of the form  $F(x_1,\dots,x_h): \{0,1\}^h \to \mathbb{R}$. A \emph{Weighted Disjunctive Normal  Form (WDNF)} formula~\cite{DBLP:conf/dna/ZhangJ05} is a PB function expressed as:

\[
F(x_1, \dots,x_h) = \sum_{k=1}^m w_k \cdot c_k(x_1, \dots, x_h)
\]

Where each $c_k$ is a cube and $w_k \in \mathbb{R}$ is its weight. 
\end{definition}

For instance, the PB formula $F(x_1, x_2, x_3) = 3(\neg x_1) + 1(\neg x_1 \land x_2) + 5(x_1 \land \neg x_2 \land x_3)$
is in WDNF. Assigning $x_1=0$, $x_2=1$, $x_3=1$ results in $F(0,1,1) = 3 + 1 = 4$. 

For a cube $c_k$, we denote by $S_k$ the set of variables in $c_k$, partitioned into positive variables $S^+_k$ and negated variables $S^-_k$. For example, in the cube $c_k \equiv x_1 \land \lnot x_2 \land x_3$, we have $S^+_k = \{x_1, x_3\}$ and $S^-_k = \{x_2\}$.

A WDNF formula, and any other PB function, can be interpreted as a game where variables are players and the formula's output is the profit~\cite{PB_shap_and_banzhaf}. 
The characteristic function (recall Def.~\ref{definition_characteristic_function}) of this game is defined by Def.~\ref{def:PB_characteristic_function}.
To find the profit of a coalition $S$, we set $x_i = 1$ for all $i \in S$, $x_i = 0$ for all $i \notin S$, and evaluate the WDNF formula. See Fig.~\ref{fig:CF_defs} for an illustration.

\begin{definition}[PB function's Characteristic Function]
\label{def:PB_characteristic_function}
Given a PB function $F(V): \{0,1\}^h \to \mathbb{R}$ over $V = \{x_1, x_2, \ldots, x_h\}$, its characteristic function $\mathcal{M}_{F(V)}(S)$ is defined for any subset of variables $S \subseteq V$ as follows:

\[
\mathcal{M}_{F(V)}(S) = F \left( \mathbf{x} = \begin{cases} 
1 & \text{if } x_i \in S \\
0 & \text{if } x_i \notin S
\end{cases} \right)
\]

This means that for each $x_i \in S$, we set $x_i = 1$ and for each $x_i \notin S$, we set $x_i = 0$. The characteristic function then evaluates $F$ under this assignment.
\end{definition}

Having defined the characteristic function for a WDNF formula $F$, we can now compute its Shapley values via Formula~\ref{eq:s}. 
In general, computing Shapley values for pseudo-Boolean functions is \#P-Hard (see Appendix~\ref{PB_calc_complexity}). A key insight of this paper is that, for WDNF, Shapley values can be computed in linear time using the following formula:
\begin{equation}
\phi_i(F) = \sum\limits_{k=1}^m w_k \times 
\begin{cases}
\frac{1}{|S^+_k|\binom{|S_k|}{|S^+_k|}} & \text{if } i \in S^+_k \\
\frac{-1}{|S^-_k|\binom{|S_k|}{|S^-_k|}} & \text{if } i \in S^-_k \\
0 & \text{if } i \notin S_k
\end{cases}
\label{shap_simplified_formula}
\end{equation}

Prior to evaluating Formula~\ref{shap_simplified_formula} and the formulas presented in Sect~\ref{sec:more_than_shap}, we remove all cubes where $S^+_k \cap S^-_k \neq \emptyset$ as such cubes are unsatisfiable. In Appendix~\ref{sec:shap_on_WDNF_appendix}, we prove that Formula~\ref{shap_simplified_formula} correctly computes the Shapley values. The proof leverages the linearity and the null player out (NPO) properties of Shapley values. 
Additionally, we show how Formula~\ref{shap_simplified_formula} can be computed in linear time and applied to Weighted Conjunctive Normal Form (WCNF) formulas~\cite{silva2021maxsat}.

\section{Beyond SHAP}
\label{sec:more_than_shap}

In this section, we present formulas that efficiently compute the Shapley interaction values $\phi_{i,j}$ (Table~\ref{table_shap_iv}), Banzhaf values $\beta_i$, and Banzhaf interaction values $\beta_{i,j}$ over WDNFs. Proofs of correctness appear Appendix~\ref{sec:interaction_values}.

Shapley interaction values measure how the interaction between two features affects the prediction~\cite{PB_shap_and_banzhaf}. One possible definition is:
\begin{definition}[Shapley Interaction Values]
\label{shap_iv_def}
Shapley interaction value of features $f_i$ and $f_j$ is the difference between the Shapley values of $f_j$ when $f_i$ always participates and when $f_i$ is always missing: 
$\phi_{i,j\;i \neq j}=\phi_{j|i=1}-\phi_{j|i=0}$
\label{definition_shap_interaction_values}
\end{definition}

By combining Formula~\ref{shap_simplified_formula} with Def.~\ref{definition_shap_interaction_values}, we can derive simple formulas for Shapley interaction values, see Table~\ref{table_shap_iv} for further details:

\begin{table}[t]
\centering
\def\arraystretch{1.5}
\begin{tabular}{l|l|l|l}
  & $i \in S^+_k$ & $i \in S^-_k$ & $i \notin S_k$ \\\hline
$j \in S^+_k$ & $\frac{w}{(|S^+_k|-1)\binom{|S_k|-1}{|S^+_k|-1}}$ & $\frac{-w}{|S^+_k|\binom{|S_k|-1}{|S^+_k|}}$ & 0 \\
$j \in S^-_k$ & $\frac{-w}{|S^-_k|\binom{|S_k|-1}{|S^-_k|}}$ & $\frac{w}{(|S^-_k|-1)\binom{|S_k|-1}{|S^-_k|-1}}$ & 0 \\
$j \notin S_k$ & 0 & 0 & 0 \\
\end{tabular}
\caption{\label{table_shap_iv} To calculate $\phi_{i,j}$ iterate through all the cubes of the WDNF formula. For each cube $c_k$ and pair of variables $i,j \in S_k$ select the appropriate cell from the table above and apply its formula.}
\end{table}

Banzhaf values satisfy three of Shapley's four properties: null player, symmetry, and linearity. However, they do not satisfy efficiency~\cite{banzhaf1965}. They also possess another useful property: the Banzhaf value of player $i$, equals the difference in the expected game's profit, under a uniform distribution over all subsets, with and without $i$:

\begin{equation} \label{banzhaf_expectations_def}
\beta_i(M) = \mathbb{E} [M(S)|i\in{S}]-\mathbb{E} [M(S)|i\notin{S}]
\end{equation}

Previous work has shown how to calculate Banzhaf values on decision trees \cite{banzhaf_for_decision_trees, shapiq}, facts in query answering \cite{abramovich2023banzhafvaluesfactsquery}, and on other tasks. The formula below enables their linear-time computation on a WDNF formula:

\begin{equation}
\beta_i(F) = \sum\limits_{k=1}^m \frac{w_k}{2^{|S_k|-1}} \times 
\begin{cases}
1 & \text{if } i \in S^+_k \\
-1 & \text{if } i \in S^-_k \\
0 & \text{if } i \notin S_k
\end{cases}
\label{banzhaf_simplified_formula}
\end{equation}


Banzhaf interaction values examine the difference in expectations when feature $i$ and $j$ are both missing/participating versus when only one participates~\cite{interaction_values}.
Given a WDNF they be computed using the formula below:

\begin{equation}
\beta_{i,j\;i \neq j}(F) = \sum\limits_{k=1}^m \frac{w_k}{2^{|S_k|-2}} \times 
\begin{cases}
1 & \text{if } i,j\in{S^+_k} \\
1 & \text{if } i,j\in{S^-_k} \\
-1 & \text{if } i\in{S^+_k} \land j\in{S^-_k} \\
-1 & \text{if } i\in{S^-_k} \land j\in{S^+_k} \\
0 & \text{otherwise }
\end{cases}
\label{banzhaf_simplified_formula_interactions}
\end{equation}

\section{Decision Pattern}
\label{sec:decision_pattern}
This section introduces the concept of a \emph{decision pattern}, central to \algname. We begin by defining decision trees and root-to-leaf paths, and then build on these foundations to define decision patterns and present the efficient \emph{CalcDecisionPatterns} algorithm for computing them.

\begin{definition}[Decision Tree]  
\label{def:decision_tree}
A \emph{decision tree} is a rooted binary tree $T = (N_T = \{L_T \cup I_T\}, E_T, r_T)$, where:  
\begin{itemize}  
    \item $r_T \in N_T$ is the \emph{root node}.     
    \item Each node $n \in N_T$ is either a childless \emph{leaf} $l \in L_T$ or an \emph{inner node} $n \in I_T$ with two children: $n.\textit{left}$ and $n.\textit{right}$.
    \item A leaf $l \in L_T$ stores an output value $w_l \in \mathbb{R}$.  
    \item An inner node $n \in I_T$ is associated with a feature $n.\textit{feature} \in \{1, \dots, h\}$ and a threshold value $\theta_n \in \mathbb{R}$.
\end{itemize}  

For a node $n$ and consumer feature values $c = (c_1, c_2, \dots, c_h) \in \mathbb{R}^h$, we define the function $n.\textit{split}(c)$:  
\[
n.\textit{split}(c) = \begin{cases}  
\text{True} & \text{if } c_{n.\textit{feature}} < \theta_n \\  
\text{False} & \text{otherwise}  
\end{cases}  
\]  

\end{definition}

\begin{definition}[Root-to-Leaf Path]
Given a decision tree $T$ and its leaf $l \in L_T$, the \emph{root-to-leaf path} of $l$ is the unique simple path from the root $r_T$ to $l$: $(n_1\equiv r_T, n_2, \dots, n_{D-1}, n_D\equiv l)$.  
\end{definition}

\begin{definition}[Decision Pattern]
\label{def:decision_pattern}
Given a decision tree $T$, its leaf $l \in L_T$, their root-to-leaf path $(n_1\equiv r_T, n_2, \dots, n_{D-1}, l)$, and consumer feature values $c = (c_1, c_2, \dots, c_h) \in \mathbb{R}^h$, the \emph{decision pattern} $p$ is a binary sequence of length $D-1$. The $i$'th bit in the sequence is:

\[ 
p[i] = 
\begin{cases} 
1 & \text{if } (n_i.\textit{split}(c) = \text{True}) \land (n_i.left=n_{i+1}) \\[6pt] 
1 & \text{if } (n_i.\textit{split}(c) = \text{False}) \land (n_i.right=n_{i+1}) \\[6pt] 
0 & \text{otherwise} 
\end{cases} 
\]

This bit indicates whether, at node $n_i$, the consumer $c$ would follow the root-to-leaf path. A value of 1 means the consumer continues along the path to $n_{i+1}$, while a value of 0 means it would branch off in a different direction.
\end{definition}

In Fig.~\ref{fig:wdnf_construction}, the `Consumer Pattern' and `Baseline Pattern' columns illustrate how the decision patterns are computed.


Let $C$ be the consumer data matrix with rows $c \in C$, and define $n.\text{split}(C) = (n.\text{split}(c))_{\forall c \in C}$.

Given a decision tree $T$ with $L$ leaves and consumer data $C$ of size $n$, \emph{CalcDecisionPatterns} (Alg.~\ref{alg:calc_decision_patterns}) traverses $T$ using Breadth-First Search (BFS), applying Def.~\ref{def:decision_pattern} at each node. Running in $O(nL)$ time, it returns a dictionary $P$ mapping each leaf $l_j \in L_T$ to its consumer decision patterns, where $P[l_j][c_i]$ stores the pattern for consumer $c_i \in C$ at leaf $l_j$.

\begin{algorithm}[t]
\caption{Mapping Each Leaf to Its Decision Patterns}
\label{alg:calc_decision_patterns}
\begin{algorithmic}[1]
\Function{CalcDecisionPatterns}{$T, C$}
    \State $P_{\textit{leaves}} \gets \{\}$ \label{line:p_init}
    \State $P_{\textit{all}} \gets \{r_T: (0)_{\forall c \in C}\}$ \label{line:p_inner_nodes_init}
    \For{$n$ in BFS(T)}: \label{line:algo_bfs}
        \If{$n$ is a leaf}
            \State $P_{\textit{leaves}}[n] = P_{\textit{all}}[n]$
        \Else
            \State $P_{\textit{all}}[n.\textit{left}] = (P_{\textit{all}}[n] << 1)
        + n.\textit{split}(C)$ 
            \State $P_{\textit{all}}[n.\textit{right}] = (P_{\textit{all}}[n] << 1)
        + \lnot n.\textit{split}(C)$ 
           \label{line:inner_node_case_end}
        \EndIf \label{line:leaf_case_end}
    \EndFor
    \State \Return $P_{\textit{leaves}}$ \label{line:return_p}
\EndFunction
\end{algorithmic}
\end{algorithm}

\section{Constructing a WDNF Representation}
\label{sec:baseline_shap}

\begin{figure*}[t]
\centering
\includegraphics[width=0.98\textwidth]{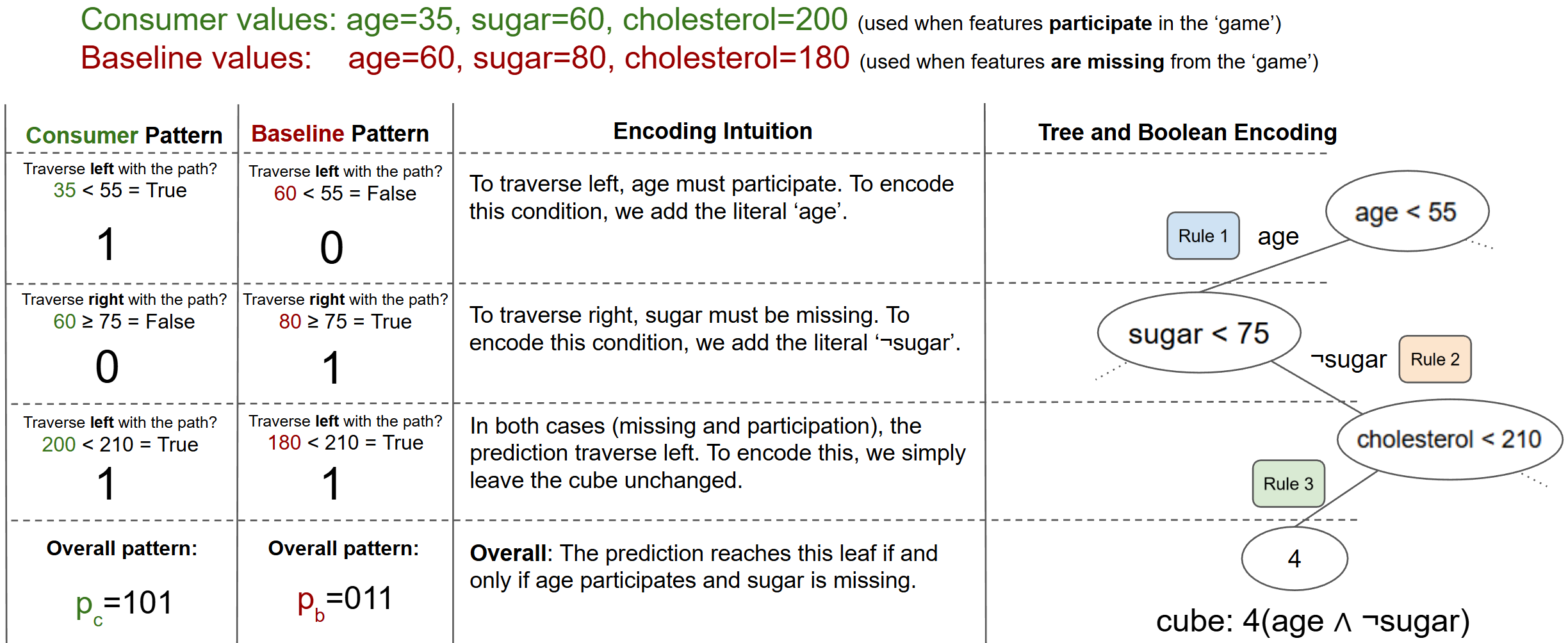} 
\caption{ An illustration of the WDNF construction process on a small example. The consumer and baseline values are shown alongside a root-to-leaf path. The table explains how a weighted cube is iteratively constructed from these inputs. To compute the Shapley value contribution of the shown leaf, apply Formula~\ref{shap_simplified_formula} to the constructed weighted cube: $4(age \land \neg sugar)$. For Banzhaf values, use Formula~\ref{banzhaf_simplified_formula}; for Banzhaf interaction values, use Formula~\ref{banzhaf_simplified_formula_interactions}; and for Shapley interaction values, use Table~\ref{table_shap_iv}.
}
\label{fig:wdnf_construction}
\end{figure*}

Our baseline SHAP algorithm constructs a WDNF formula $F$ representing the model’s characteristic function. An illustration of this construction is provided in Fig.~\ref{fig:wdnf_construction}. The variables in $F$ correspond to model features, and each cube captures the contribution of a single leaf. A variable $f_i$ set to 1 indicates the feature is present (i.e., set to the consumer’s value $c[f_i]$), while 0 denotes a missing feature (i.e., set to the baseline value $b[f_i]$). A cube is satisfied if and only if the prediction reaches its corresponding leaf:

{\fontsize{9}{10}\selectfont
\begin{equation}
\label{model_to_wdnf_equation}
\model \left(  \mathbf{f} = \begin{cases} 
c[f_i] & \text{if } f_i \in S \\
b[f_i] & \text{if } f_i \notin S
\end{cases} \right) = F \left( \mathbf{x} = \begin{cases} 
1 & \text{if } f_i \in S \\
0 & \text{if } f_i \notin S
\end{cases} \right)
\end{equation}
}

We present four simple rules for constructing the WDNF formula. Our goal is to construct a cube that represents a leaf $l$, a consumer $c$, and a baseline $b$. We first run \emph{CalcDecisionPatterns} to obtain the consumer decision pattern $p_c$ and baseline decision pattern $p_b$. Then, using the path, $p_c$, $p_b$, and the four rules below, we construct the cube. The rules are applied from the tree root ($i=1$) down to the leaf’s parent node ($i=D-1$):

\begin{enumerate}
    \item If $p_c[i] = 1$ and $p_b[i] = 0$: The prediction reaches $n_{i+1}$ only when the consumer value is used—i.e., when $f_i$ (that is, $n_i.\mathit{feature}$) participates. Add $f_i$ to the cube.
    
    \item If $p_c[i] = 0$ and $p_b[i] = 1$: The prediction reaches $n_{i+1}$ only when the baseline value is used—i.e., when $f_i$ is missing. Add the literal $\neg f_i$ to the cube.
    
    \item If $p_c[i] = 1$ and $p_b[i] = 1$: The prediction always reaches $n_{i+1}$. Leave the cube unchanged.
    
    \item If $p_c[i] = 0$ and $p_b[i] = 0$: The prediction never reaches $n_{i+1}$. Set the cube to $\bot$ (unsatisfiable).
\end{enumerate}

Since the number of decision patterns is limited, we can precompute the cube for every possible input. The \emph{MapPatternsToCube} function (Alg.~\ref{alg:fill_wdnf_table}) takes the list of features along a root-to-leaf path and applies the four rules above to map each pair of consumer and baseline patterns ($p_c$ and $p_b$) to the corresponding cube.

\begin{algorithm}[t]
\caption{Decision Patterns to Cube Mapping}
\label{alg:fill_wdnf_table}
\begin{algorithmic}[1]
\Function{MapPatternsToCube}{$\textit{features}$}
    \State $d \gets\{ 0 \mapsto \{ 0 \mapsto (\emptyset, \emptyset) \} \}$
    \For{$f  \in \textit{features}$}
        \State $d_{old} \gets d$
        \State $d \gets \{\}$
        \For{$p_c$ in $d_{old}$}
            \For{$p_b$ in $d_{old}[p_c]$}
                \State $(S^+, S^-) \gets d_{old}[p_c][p_b]$ \label{line:extract_splus_sminus}
                \State $d[2p_c + 1][2p_b + 0] \gets (S^+ \cup \{f\}, S^-)$ \label{line:four_rules_start}
                \State $d[2p_c + 0][2p_b + 1] \gets (S^+, S^- \cup \{f\})$
                \State $d[2p_c + 1][2p_b + 1] \gets (S^+, S^-)$
            \EndFor
        \EndFor
    \EndFor
    \State \Return $d$
\EndFunction
\end{algorithmic}
\end{algorithm}

\section{A Generic Baseline SHAP Implementation}
\label{sec:baseline_shap_impl}
Let $P[l][x]$ denote the decision pattern of consumer or baseline values $x$ on leaf $l$, computed using \emph{CalcDecisionPatterns}. Let $d_l$ be the mapping for leaf $l$ built using \emph{MapPatternsToCube}. Using $P$ and $d_l$, one can now compute the \emph{Baseline SHAP}. The WDNF of a decision tree $T$, a consumer $c$, and baseline values $b$ can be calculated using Formula~\ref{eq_generalized_baseline_shap}, which constructs a WDNF by aggregating the cubes of all leaves, each weighted by its corresponding leaf weight.

\begin{equation}
\label{eq_generalized_baseline_shap}
\sum_{l \in L_T} w_l \cdot d_l[\:P[l][c]\:][ \: P[l][b] \:]
\end{equation}

We apply the linear-time Shapley values formula (Formula~\ref{shap_simplified_formula}) to the resulting WDNF to compute the desired baseline SHAP. Similarly, this WDNF can be used to compute Shapley interaction values, Banzhaf values, or Banzhaf interaction values by leveraging Table~\ref{table_shap_iv}, Formula~\ref{banzhaf_simplified_formula}, or Formula~\ref{banzhaf_simplified_formula_interactions}, respectively.

\section{Efficient Background SHAP Equation}
\label{sec:bg_eq}
\emph{Background SHAP} takes three inputs: a decision tree $T$ with depth $D$ and $L$ leaves; consumer data $C$ with $n$ rows; and background data $B$, a matrix with $m$ rows of baseline feature values.
Unlike Baseline SHAP, which uses one fixed baseline for missing features, Background SHAP averages the Shapley values across all baselines in the background data, providing more accurate results. 

Formula~\ref{trival_bg} computes Background SHAP for a single consumer $c \in C$ and a single feature $i$ in $O(m)$ (assuming $L$ is constant). It uses the baseline WDNF formula (Formula~\ref{eq_generalized_baseline_shap}) and the linear-time Shapley values formula (Formula~\ref{shap_simplified_formula}).  

\begin{equation}
\label{trival_bg}
\phi_i (\frac{1}{|B|}\sum_{b_k \in B} \sum_{l \in L_T} w_l \cdot d_l[\:P[l][c]\:][ \: P[l][b_k] \:])
\end{equation}
Using Formula~\ref{trival_bg}, computing Background SHAP for all $n$ consumers in $C$ takes $O(nm)$ time. We derive a new $O(n + m)$ formula that leverages GPU-friendly matrix multiplication. The derivation is detailed below:

{\fontsize{9}{10}\selectfont
\begin{equation}
\label{eq:bg_shap_derivation}
\begin{aligned}
\phi_i ( \frac{1}{|B|}\sum_{b_k \in B} \sum_{l \in L_T} w_l \cdot d_l[\:P[l][c]\:][ \: P[l][b_k] \:]) &\stackrel{(a)}= \\
\frac{1}{|B|}\sum_{l \in L_T} \sum_{b_k \in B} w_l \cdot \phi_i ( d_l[\:P[l][c]\:][ \: P[l][b_k] \:]) &\stackrel{(b)}= \\
\frac{1}{|B|}\sum_{l \in L_T} \sum_{p_b \in \{F, T\}^{D-1}} vc_{l,p_b} \cdot w_l \cdot \phi_i ( d_l[\:P[l][c]\:][p_b]) &\stackrel{(c)}= \\
\frac{1}{|B|}\sum_{l \in L_T} \sum_{p_b \in \{F, T\}^{D-1}} \vc[l][p_b] \cdot w_l \cdot \phi_i ( d_l[\:P[l][c]\:][p_b]) &\stackrel{(d)}= \\
\sum_{l \in L_T} w_l \cdot (\sum_{p_b \in \{F, T\}^{D-1}} \frac{\vc[l][p_b]}{|B|} \cdot \phi_i ( d_l[\:P[l][c]\:][p_b])) &\stackrel{(e)}= \\
\sum_{l \in L_T} (w_l \cdot \mathbf{M_{l,i}} \cdot \mathbf{f_l})[\:P[l][c]\:] &\stackrel{(f)}= \\
\sum_{l \in L_T} \mathbf{s_{l,i}}[\:P[l][c]\:] &\stackrel{(g)}= \\ 
\end{aligned}
\end{equation}
}

We explain each transformation step below:

\begin{enumerate}[label=\alph*)]
    \item This is Formula~\ref{trival_bg}.

    \item Follows from the linearity property of the Shapley value: Functions of the form $f_1 = f_2 + w \cdot f_3$ satisfy $\phi_i(f_1) = \phi_i(f_2) + w \cdot \phi_i(f_3)$. We also reorder the summations.\label{stepb}

    \item Mark $vc_{l,p_b} = |\{ b_k \in B \mid P[l][b_k] = p_b \}|$. Since the summands depend only on the decision patterns, we can rewrite the summation by looping over all possible baseline decision patterns. For each pattern, we multiply by the number of background instances that match it.

    \item At the start of the algorithm, we precompute the number of baselines matching each pattern in $O(mL)$ time:\\
    $\vc = \emph{CalcDecisionPatterns}(T, B).\emph{value\_counts()}$\\
    where \emph{value\_counts} is a standard function from the \texttt{pandas} Python package. For each leaf $l$ and pattern $p_b$, we have $vc_{l,p_b}\!=\! \vc[l][p_b]\!=\!|\{ b_k \in B \mid P[l][b_k] = p_b \}|$. 
    
    This precomputation reduces the summation complexity of the formula from $O(nm)$ to $O(n+m)$.

    \item Simple arithmetic: We push $\frac{1}{|B|}$ into the inner summation and pull $w_l$ out.

  \item The inner summation becomes a matrix-vector product:

Let $\mathbf{f_l}$ be the frequency vector of the background patterns, where $\mathbf{f_l}[p_b] = \frac{\vc[l][p_b]}{|B|}$, i.e., the relative frequency of baseline pattern $p_b$ at leaf $l$.

Let $\mathbf{M_{l,i}}$ be the Shapley matrix, where $\mathbf{M_{l,i}}[p_c][p_b] = \phi_i(d_l[p_c][p_b])$, representing the contribution of leaf $l$ to feature $i$’s Shapley value for the pair $(p_c, p_b)$, assuming the leaf weight is 1.

The multiplication $w_l \cdot \mathbf{M_{l,i}} \cdot \mathbf{f_l}$ sums all the effects that leaf $l$ has on the Shapley values of player $i$ across all baseline decision patterns, weighted by their frequency. The result is a vector representing the Shapley value effects for all consumer patterns. 

Extracting $(w_l \cdot \mathbf{M_{l,i}} \cdot \mathbf{f_l})[P[l][c]]$ yields the Shapley value effect for consumer $c$ at leaf $l$.

\label{stepf}
  \item Since the vector $w_l \cdot \mathbf{M_{l,i}} \cdot \mathbf{f_l}$ is independent of the consumer $c$, we can precompute it for every leaf $l$ and feature $i$. We denote this vector by $\mathbf{s_{l,i}}$. After these vectors are built, computing SHAP values per consumer only requires fetching, for each leaf $l$, the element in $\mathbf{s_{l,i}}$ indexed by the consumer’s decision pattern at $l$. Computing SHAP values for all features takes $O(nLD)$ time, since each root-to-leaf path involves at most $D$ features.
\end{enumerate}

The derivation above uses only the linearity property of Shapley values (see step~\ref{stepb}. Therefore, it holds for any metric that satisfies linearity, including Shapley interaction values, Banzhaf values, and Banzhaf interaction values.

\section{Efficient Path-Dependent SHAP Equation}
\label{sec:pd_eq}
Instead of computing the frequencies using a background dataset, Path-Dependent SHAP estimates them using the nodes cover property (i.e. the number of training samples that reached the node during training). 

Path-Dependent SHAP can be computed by simply replacing the Background frequency vector $f_l$ with the Path-Dependent frequency vector $f_{l_{pd}}$ in Formula~\ref{eq:bg_shap_derivation}(\ref{stepf}. Given a decision tree $T$, a leaf $l$ with its root-to-leaf path $(n_1 \equiv r_T, n_2, \dots, n_{D-1}, n_D \equiv l)$, and a baseline decision pattern $p_b$, the vector $f_{l_{pd}}$ is computed as follows:

\begin{equation}
\label{eq:path_dependent_estimation}
\mathbf{f_{l_{pd}}}[p_b] = \prod_{i=1}^{D-1} \begin{cases}
\frac{n_{i+1}.cover}{n_i.cover} & \text{if } p_b[i] = 1 \\[0.5em]
1 - \frac{n_{i+1}.cover}{n_i.cover} & \text{if } p_b[i] = 0 \\
\end{cases}
\end{equation}

For example, the frequency of the pattern $5$ (binary $101$) along the root-to-leaf path $(n_1 \equiv r_T, n_2, n_3, n_4 \equiv l)$ is:

\[
\mathbf{f_{l_{pd}}}[5] =
\frac{n_2.cover}{n_1.cover} \cdot (1 - \frac{n_3.cover}{n_2.cover}) \cdot \frac{n_4.cover}{n_3.cover}
\]  

\section{\algname Algorithm}
\label{sec:woodelf_alg}

We are now ready to present our main algorithm, \algname, shown in Alg.~\ref{alg:woodelf_code}. 
\algname takes as input a decision tree $T$, consumer data $C$, background data $B$ (empty $B$ means Path-Dependent SHAP), and a function $v$. 

The function $v$ takes a cube where each variable represents a feature, e.g. $(\text{age} \land \lnot \text{sugar})$, and returns a mapping from feature subsets (of size one or more) to real numbers. For instance, $v$ can compute Shapley values for individual features, as well as interaction values for feature pairs.

\algname outputs Path-Dependent or Background (depending on whether $B$ is empty) Shapley/Banzhaf values or interaction values (depending on $v$) on the decision tree $T$ for the given consumers. To compute values for a decision tree ensemble, one simply runs \algname on each tree and sums the results. Correctness follows from the linearity property of both Shapley and Banzhaf values.

\begin{algorithm}[t]
\caption{An Efficient SHAP and Banzhaf algorithm}
\label{alg:woodelf_code}
\begin{algorithmic}[1]
\Function{Woodelf}{$T, C, B, v$}
    \Statex \Comment{Step 1, compute $f$}
    \If{$|B| > 0$} \label{line:f_begin}
    \Comment{Background}
    \State $P_b = \textit{CalcDecisionPatterns}(T, B)$
    \State $f = P_b.\textit{value\_counts}(\textit{normalize}=True)$ \label{line:value_counts}
    \Else 
    \Comment{Path Dependent}
    \State Compute $f$ using $T$ and Formula~\ref{eq:path_dependent_estimation}
    \EndIf \label{line:f_end}

    \Statex \Comment{Step 2, compute $M$}
    \State $M = \{\}$ \label{line:M_start}
    \For{$l \in L_T$}
        \State $path = root\_to\_leaf\_path(T, l)$
        \State $path\_features = (n.\textit{feature})_{\forall n \in path}$
        \State $d_l = \textit{MapPatternsToCube}(path\_features)$
        \For{$p_c$ in $d_l$}
            \For{$p_b$ in $d_l[p_c]$}
                \State $cube = d_l[p_c][p_b]$
                \For{$feature, value$ in $v(cube)$}
                    \State $M[l][feature][p_c][p_b] = value$
                \EndFor
            \EndFor
        \EndFor
    \EndFor \label{line:M_end}

    \Statex \Comment{Step 3, compute $s$} 
    \State $s = \{\}$ \label{line:s_start}
    \For{$l \in L_T$}
        \For{$feature$ in $M[l]$}
            \State $s[l][feature] = w_l \cdot M[l][feature] \cdot f[l]$ \label{line:matrix_mult}
        \EndFor
    \EndFor \label{line:s_end}

    \Statex \Comment{Step 4, compute the actual values}
    \State $P_c = \textit{CalcDecisionPatterns}(T, C)$ \label{line:final_computation_start}
    \State $values = \{\}$
    \For{$l \in L_T$}
        \For{$feature$ in $s[l]$}
            \State $values[feature] \mathrel{+}= s[l][feature][\:P_c[l] \:]$ \label{line:np_indexing}
        \EndFor
    \EndFor \label{line:final_computation_end}

    \State \Return $values$
\EndFunction
\end{algorithmic}
\end{algorithm}

The algorithm uses the equations from Sect.~\ref{sec:bg_eq} and~\ref{sec:pd_eq}. Lines~\ref{line:f_begin}–\ref{line:f_end} compute the frequency vector $f$ for either Background or Path-Dependent SHAP. Lines~\ref{line:M_start}–\ref{line:M_end} compute the contribution matrix $M$. Lines~\ref{line:s_start}–\ref{line:s_end} use $f$ and $M$ to compute $s$, the vector mapping consumer patterns to contributions. Finally, lines~\ref{line:final_computation_start}–\ref{line:final_computation_end} use $s$ to compute the desired Shapley/Banzhaf values  or interaction values.

\subsection{Algorithmic Improvements and Complexity}
\label{sec:speed_ups}

Our actual implementation is more advanced than the version shown in Alg.~\ref{alg:woodelf_code}. It incorporates several key optimizations that substantially reduce the algorithm's runtime — with the first even improving its theoretical complexity:

\begin{enumerate}
    \item Each matrix $M[l][\text{feature}]$ has size $4^D$ (recall that $D$ is the depth of the tree), since both $p_c$ and $p_b$ can take any value between $0$ and $2^D$. Furthermore, the dictionary returned by \emph{MapPatternsToCube} has size $3^D$, as the number of cubes triples at each step. This means the matrix $M[l][\text{feature}]$ is sparse, with at most $3^D$ non-zero entries. By using sparse matrix multiplications, we reduce the complexity of line~\ref{line:matrix_mult} from $O(4^D)$ to $O(3^D)$, thereby improving \algname’s overall complexity (see Table ~\ref{complexity_table}).
    
\item The function \emph{MapPatternsToCube} and the matrix $M[l][\text{feature}]$ depend solely on the features repeated along the root-to-leaf path and the path’s length. For example, all leaves at depth 6 with unique features share the same matrices. We exploit this by using a caching mechanism, which significantly reduces the computations in lines~\ref{line:M_start}--\ref{line:M_end}.

    \item For every consumer/baseline $x$, the decision pattern of neighboring leaves $l_i$ and $l_{i+1}$ ($\exists n$ s.t.~$n.left = l_i$, $n.right = l_{i+1}$) differ only in the last bit (see Def.~\ref{def:decision_pattern}).  
    We leverage this property to accelerate lines~\ref{line:value_counts} and~\ref{line:np_indexing}.
    
  \item We only need to compute half of the Shapley/Banzhaf interaction values because, for all $i, j$, $\phi_{i,j} = \phi_{j,i}$.
    
  \item The length of each decision pattern is limited by the tree’s depth. In the \emph{CalcDecisionPatterns} algorithm, we select the appropriate unsigned integer type (e.g., \texttt{uint8}, \texttt{uint16}, \texttt{uint32}) based on the tree’s maximum depth. This reduces compute time by enabling more efficient use of SIMD.

  \item Line~\ref{line:np_indexing} utilizes vectorized NumPy indexing. It treats $P_c[l]$ as a series of indices and returns a series of the corresponding elements from $s[l][feature]$.

\end{enumerate}

\begin{table}[b]
\centering
\def\arraystretch{1.3}
\setlength{\tabcolsep}{1mm}
{\fontsize{9}{10}\selectfont
\begin{tabular}{l|l|l}
Task & \algname & State-of-the-art \\\hline
PD & $O(nTLD\!+\!TL3^D\!D)$ & $O(nTLD\!+\!TL2^D\!D)$\\
BG & $O(mTL\!+\!nTLD\!+\!TL3^D\!D)$ & $O(mTL\!+\!nT3^D\!D)$\\
PDIV & $O(nTLD^2\!+\!TL3^D\!D^2)$ & $O(nTLD^2)$\\
BGIV & $O(mTL\!+\!nTLD^2\!+\!TL3^D\!D^2)$ & $O(mTL\!+\!nT3^D\!D^2)$\\
\end{tabular}}
\caption{\label{complexity_table} Complexity results. Legend: PD = Path-Dependent SHAP, BG = Background SHAP, BGIV/PDIV = Calculation of all Shapley interaction values, $n=|C|$, $m=|B|$, $T$ = number of trees, $L$ = leaves per tree, and $D$ = tree depth.}
\end{table}

Table~\ref{complexity_table} summarizes the complexity of \algname, with detailed analysis in Appendix~\ref{sec:complexity_analysis}. The state-of-the-art Path-Dependent SHAP algorithm is FastTreeShap, while PLTreeShap is the state-of-the-art for Background SHAP; both outperform the $\textsc{shap}$ Python package.

\algname improves on PLTreeShap’s complexity when $L < n$ by leveraging a core step (lines~\ref{line:M_start}--\ref{line:matrix_mult}) whose cost is independent of dataset size—a key factor behind the empirical gains shown in the next section. However, this step might become a bottleneck for very deep trees or small datasets, where PLTreeShap and the \textsc{shap} Python package may outperform \algname.

\begin{table*}[t]
\centering
\def\arraystretch{1.4}

\hspace{2cm} \textbf{IEEE-CIS} 
\newline
\begin{tabularx}{\textwidth}{l|X|r|r|r|r|r}
\textbf{Task} & \textbf{SOTA CPU Algorithm} & \textbf{\textsc{shap} package} & \multicolumn{2}{c|}{\textbf{SOTA}} & \multicolumn{2}{c}{\textbf{\algname}} \\
             &&& \textbf{CPU} & \textbf{GPU} & \textbf{CPU} & \textbf{GPU} \\\hline

Path-Dependent SHAP & FastTreeShap v2 & 151 sec & 16 sec & 0.9 sec & 6 sec & 3.3 sec \\
Background SHAP & PLTreeShap & 10 days* & 245 sec & 14 hours* & 12 sec & 10 sec \\
Path-Dependent SHAP IV & FastTreeShap v1 & 33 hours* & 350 sec & 105 sec* & 11 sec & 8 sec \\
Background SHAP IV & PLTreeShap & X & 597 sec* & X & 19 sec & 12 sec \\
\end{tabularx}

\vspace{0.5em}

\textbf{KDD Cup 1999}
\begin{tabularx}{\textwidth}{l|X|r|r|r|r|r}
\textbf{Task} & \textbf{SOTA CPU Algorithm} & \textbf{\textsc{shap} package} & \multicolumn{2}{c|}{\textbf{SOTA}} & \multicolumn{2}{c}{\textbf{\algname}} \\
             &&& \textbf{CPU} & \textbf{GPU} & \textbf{CPU} & \textbf{GPU} \\\hline

Path-Dependent SHAP & FastTreeShap v2 & 51 min & 373 sec & 7.9 sec & 96 sec & 3.3 sec \\
Background SHAP & PLTreeShap & 8 years* & 44 min & 3 months* & 162 sec & 16 sec \\
Path-Dependent SHAP IV & FastTreeShap v1 & 8 days* & 221 min* & 229 sec* & 193 sec & 6 sec \\
Background SHAP IV & PLTreeShap & X & 105 min* & X & 262 sec & 19 sec \\
\end{tabularx}

\caption{\label{performance_table} 
Performance comparison between the \textsc{shap} Python package, the state-of-the-art (SOTA) methods and \algname. The ‘SOTA CPU Algorithm’ column lists the best known CPU algorithm for each task, and the ‘SOTA, CPU’ column shows its runtime. The SOTA GPU algorithm for all tasks is GPUTreeSHAP. 
''SHAP'' refers to computing the Shapley values of all features, while ''SHAP IV'' refers to computing all Shapley interaction values. 
Values marked with * are estimates. Estimation was necessary due to RAM limitations, long runtimes, and implementation constraints. Notably, the \textsc{shap} Python package currently supports background datasets of up to 100 rows, implicitly using only the first 100 rows of larger datasets. See Appendix~\ref{sec:estimation} for details on the estimation method. 'X' means there is no available implementation for this task.}

\end{table*}

\section{Experimental Results}
\label{sec:experimental_results}

We implemented the \algname\ Python package, which includes our algorithm. The notebooks used in the experiments are provided in the \algname Experiments repository. Detailed setup and empirical validation of the algorithm’s correctness appear in Appendix~\ref{sec:experimental_results_appendix}.

We compared \algname performance with that of the \shap package and the relevant SOTA algorithms.
To evaluate \algname at scale, we selected two of the largest and well-known tabular datasets.

The IEEE-CIS fraud detection dataset from the Kaggle competition is widely recognized, with related studies including~\cite{frauddatasetwork1, frauddatasetwork2, frauddatasetwork3, frauddatasetwork5, frauddatasetwork6}. In IEEE-CIS, $|B| = \num{118108}$, $|C| = \num{472432}$, and $F = \num{397}$ (after applying one-hot encoding to categorical features in both datasets, where $F$ denotes the total number of features after preprocessing).

The KDD Cup 1999 dataset serves as a well-established benchmark for network intrusion detection research, with related studies including~\cite{kdd_cup_winner, kdd_cup_data_review, recent_deep_learning_work_on_kdd}. In KDD Cup: $|B|=\num{4898431}, \; |C|=\num{2984154}, \; F=\num{127}$. 

In both cases, we trained an XGBoost regressor with 100 trees of depth 6 (XGBoost's default \textit{max\_depth}). All  algorithms were run sequentially without parallelization.

Our experiments were conducted in Google Colab's CPU environment with the additional RAM option enabled, utilizing 50GB of RAM instead of the standard 12GB. The GPU execution used the A100 GPU Colab runtime type. 

Table~\ref{performance_table} shows that \algname outperforms the state-of-the-art in all tasks—except GPU Path-Dependent SHAP on IEEE, where runtimes are already short. For Background SHAP, \algname achieves speed-ups of 24$\times$, 50$\times$, 165$\times$, and 333$\times$ on GPU, and 16$\times$, 20$\times$, 31$\times$, and 24$\times$ on CPU, relative to the best method on any hardware platform.

A striking example of historical improvement is Background SHAP on the KDD dataset. In 2020,~\cite{explainable_ai_trees} introduced the first polynomial-time algorithm for this task, but its quadratic complexity would still require an estimated 8 years on this dataset. Two years later,~\cite{GPUTreeShap} proposed a GPU-based implementation, reducing runtime to 3 months. In 2023,~\cite{linear_background_shap} achieved a breakthrough with a linear-time method, cutting runtime to 44 minutes. Our \algname algorithm completes the task in just 162 seconds on CPU and 16 seconds on GPU. Over just five years, the runtime has been reduced from 8 years to mere seconds!

\section{Conclusion}
\label{sec:conclusion}

We introduced \algname, a fast, unified, and GPU-friendly SHAP algorithm leveraging a novel connection between decision trees and Boolean logic. On the evaluated datasets, it outperformed state-of-the-art Background SHAP methods by 16--31$\times$ on CPU and 24--333$\times$ on GPU.

\algname provides a unified framework for model interpretability, supporting a range of attribution metrics (e.g., Banzhaf values) across different characteristic function definitions (e.g., Path-Dependent).
With its efficiency and flexibility, \algname lays a solid foundation for future research into more advanced and precise interpretability methods.

\bibliography{sample}

\appendix

\section{Background SHAP is the Most Accurate}
\label{sec:background_shap_is_the_most_accurate}
We now examine the example in Fig.~\ref{fig_simple_tree} to illustrate why Background SHAP is the most accurate approach. Details on the three common approaches analyzed in this appendix are provided in Sect.~\ref{sec:feature_importance_using_shap}. Suppose the decision tree $M$ in Fig.~\ref{fig_simple_tree} was trained on the dataset $X$.

\begin{figure}[h!]
\centering
\includegraphics[width=0.9\columnwidth]{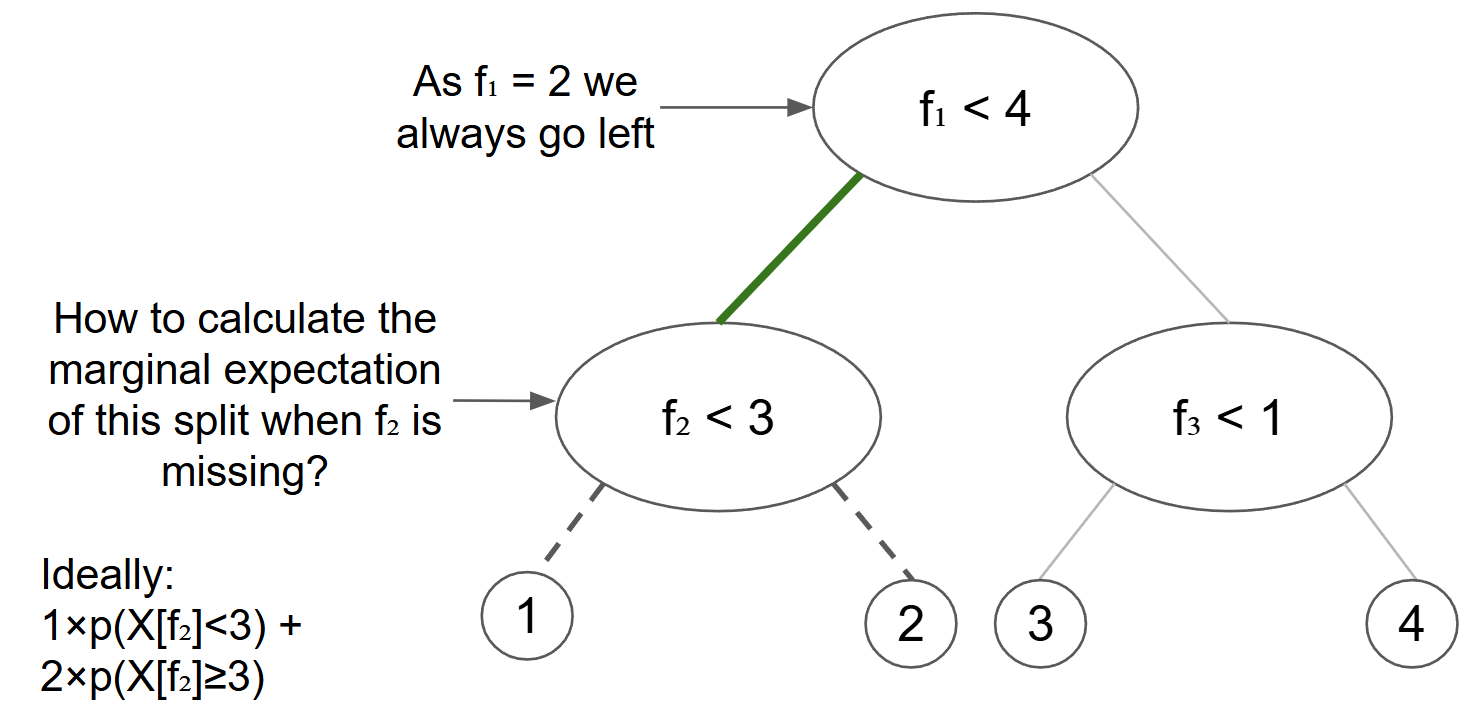}
\caption{A simple decision tree illustrating the excepted prediction computation, given $f_1 = 2$ and $f_2$ is missing.}
\label{fig_simple_tree}
\end{figure}

As part of a characteristic function definition, we seek to determine the model’s prediction when $f_1=2$, $f_3=5$ and $f_2$ is missing. Ideally, this prediction should be the marginal expectation $\mathbb{E}[M(f_1,f_2,f_3) \mid do(f_1=2, f_3=5)]$ (For more on marginal vs. conditional SHAP, see \cite{true_model_true_data}). 

At the split $f_1 < 4$, we go left since $f_1=2$. The question is how to handle the subsequent split $f_2 < 3$ when $f_2$ is missing. 

To proceed, we must estimate the probability that the condition holds, denoted $p(f_2 < 3)$.
We then use this probability to compute the expected prediction at this node: 
$\mathbb{E}[M(f_1,f_2,f_3) \mid do(f_1=2, f_3=5)] = 1 \cdot p(f_2 < 3) + 2 \cdot (1 -p(f_2 < 3)) $.

\emph{Baseline SHAP} can not estimate the full probability as it only has a single baseline value for $f_2$. \emph{Path-Dependent SHAP} estimates this probability using the cover property, which is based on the training data that reached the split node in the training process. Since only rows with $X[f_1] < 4$ reached this node, the estimate is conditioned on $X[f_1] < 4$, leading to potential bias when $f_1$ and $f_2$ are correlated. \emph{Background SHAP} can use the full training set to compute this probability accurately. 

\section{Shapley and Banzhaf values on WDNF and WCNF - With Correctness Proofs}
\label{sec:shap_on_WDNF_appendix}

This appendix provides a deeper analysis of the Shapley and Banzhaf definitions for pseudo-Boolean functions and their computation over WDNF and WCNF formulas.
We formally prove the correctness of Formulas~\ref{shap_simplified_formula} and~\ref{banzhaf_simplified_formula}, and derive their asymptotic computational complexity.

\subsection{Shapley and Banzhaf values}

\paragraph{Shapley values} are defined by Formula~\ref{eq:s}. They are known to be the unique method that satisfies four desirable properties \cite{original_shapley_paper}:

\begin{enumerate}
    \item \textbf{Efficiency}: Summing the Shapley values for all players equals the characteristic function $M$ output when all the features exist ($M(N)$) minus the output when all are missing ($M(\{\})$). Given a group of N features: $\sum^{N}_{i=1}\phi_i(M) = M(N) - M(\{\})$
    \item \textbf{Null Player}: A player who does not influence any coalition gets a Shapley value of zero.
    \item \textbf{Symmetry} (or equal treatment of equals): If players i and j contribute equally to any subset, their Shapley values are equal.
    \item \textbf{Linearity}: The Shapley value is linear, meaning that for any two characteristic functions $M_1$ and $M_2$, the Shapley value of their sum equals the sum of their Shapley values. That is, for a player $i$ and $M = M_1 + M_2$, we have:
    $\phi_{i}(M) = \phi_i(M_1) + \phi_i(M_2)$
\end{enumerate}

\paragraph{Banzhaf values} \cite{banzhaf1965} originate from voting theory. Their formula is similar to the Shapley values formula but gives the same weight to all subsets:
\begin{equation} \label{banzhaf_combinatorial_def}
\beta_i(M) = \sum\limits_{S \subseteq N\setminus \{i\}} \frac{1}{2^{(|N|-1)}} (M(S\cup\{i\}) - M(S))
\end{equation}

See Sects.~\ref{sec:shap_intro} and~\ref{sec:more_than_shap} for additional background on the Shapley and Banzhaf values.
Both Banzhaf and Shapley are computed over the game's characteristic function $M$. We have demonstrated how a pseudo-Boolean (PB) function $F$ can be treated as a game's characteristic function (see Def.~\ref{def:PB_characteristic_function}). Below are the Shapley and Banzhaf values formulas for a variable (player) in a PB function, derived by applying our characteristic function definition to the standard Shapley and Banzhaf value formulas. Here, $x_{-i} = \{x_1, \dots, x_h\} \setminus \{x_i\}$ (the set of all variables excluding $x_i$), $h$ is the total number of variables, and $\sum_{x_j \in x_{-i}} x_j$ is the number of participating players in $x_{-i}$. 

The Shapley values formula is:

\begin{equation}
\begin{aligned}
\phi_i(F) = \textstyle\sum\limits_{x_{-i} \in \{0,1\}^{h-1}} \frac{(\sum\limits_{x_j \in x_{-i}}x_j)!(h - (\sum\limits_{x_j \in x_{-i}}x_j) - 1)!}{h!} \times \Delta F, \\
\Delta F = F(x_{-i}, x_i=1) - F(x_{-i}, x_i=0).
\end{aligned}
\label{PB_shap}
\end{equation}

The Banzhaf values formula is:

\begin{equation}
\begin{aligned}
\beta_i(F) = \textstyle\sum\limits_{x_{-i} \in \{0,1\}^{h-1}} \frac{1}{2^{h-1}} \times \Delta F, \\
\Delta F = F(x_{-i}, x_i=1) - F(x_{-i}, x_i=0).
\end{aligned}
\label{PB_banzhaf}
\end{equation}

Direct calculation of these formulas is computationally expensive due to the exponential size of $\{0,1\}^{h-1}$. In this section, we show that this problem is \#P-hard in general. A key insight of our research is that Shapley and Banzhaf values for PB functions, represented in the well-known normal forms WDNF and WCNF, can be computed in linear time.

\subsection{WCNF}

The definition of Weighted Conjunctive Normal (WCNF) Form is very similar to the WDNF definition (see Def~\ref{def:WDNF}). A \emph{clause} is a disjunction of literals. For a clause $c_k$, we denote by $S_k$ the set of variables in $c_k$, partitioned into positive variables $S^+_k$ and negated variables $S^-_k$. For example, in the clause $c_k \equiv x_1 \lor \lnot x_2 \lor x_3$, we have $S^+_k = \{x_1, x_3\}$, $S^-_k = \{x_2\}$ and $S_k = S^+_k \cup S^-_k$. 

\begin{definition}[Weighted Conjunctive Normal Form (WCNF)~\cite{silva2021maxsat}]
\label{def:WCNF}
\emph{WCNF} formulas are pseudo-Boolean functions of the form:

\[
F(x_1, \dots,x_h) = \sum_{k=1}^{m} w_k \cdot c_k(x_1, \dots,x_h)
\]

Where each $c_k$ is a \textbf{clause} and $w_k \in \mathbb{R}$ is its weight. 
\end{definition}

For example, the PB function $F(x_1, x_2, x_3) = 3(\neg x_1) + 1(\neg x_1 \lor x_2) + 5(x_1 \lor \neg x_2 \lor x_3)$ is in WCNF. Assigning $x_1=0$, $x_2=1$, and $x_3=1$ results in $F(0,1,1) = 3 + 1 + 5 = 9$.

\subsection{Linear Time Shapley and Banzhaf Formulas}

Shapley and Banzhaf values can be computed in linear time on WDNF and WCNF representations using Formulas~\ref{shap_simplified_formula} and~\ref{banzhaf_simplified_formula}, respectively. Before applying these formulas, remove any cubes or clauses that contain both a variable and its negation — that is, remove any cube or clause where $|S^+_k \cap S^-_k| \ge 1$. This trivial preprocessing step is correct because such cubes are always false, and such clauses are always true. These cubes/clauses are also removed before the interaction values computation.

See tables \ref{banzhaf_calc_example} and \ref{shapley_calc_example} for Banzhaf and Shapley calculation examples:

\begin{table}[h]
\centering
\renewcommand{\arraystretch}{1.5}
\resizebox{\columnwidth}{!}{
\begin{tabular}{l|c|c|c|c}
 & $3(\neg x1) $ & $5(x_1 \land \neg x_3)$ & $2(x_2 \land x_3 \land \neg x_1)$ & total value \\\hline
$\beta_1$ (Banzhaf for $x_1$)  & $\frac{-3}{2^0}=-3$ & $\frac{5}{2^1}=2.5$ & $\frac{-2}{2^2}=-0.5$ & -1 \\
$\beta_2$ (Banzhaf for $x_2$)  & 0 & 0 & $\frac{2}{2^2}=0.5$ & 0.5 \\
$\beta_3$ (Banzhaf for $x_3$)  & 0 & $\frac{-5}{2^1}=-2.5$ & $\frac{2}{2^2}=0.5$ & -2 \\
\end{tabular}}
\caption{\label{banzhaf_calc_example}Detailed Banzhaf values calculation on $\psi = 3(\neg x1)+5(x_1 \land \neg x_3)+2(x_2 \land x_3 \land \neg x_1)$. }
\end{table}

\begin{table}[h]
\centering
\renewcommand{\arraystretch}{1.5}
\resizebox{\columnwidth}{!}{
\begin{tabular}{l|c|c|c|c}
 & $3(\neg x_1) $ & $5(x_1 \land \neg x_3)$ & $2(x_2 \land x_3 \land \neg x_1)$ & total value \\\hline
$\phi_1$ (Shapley for $x_1$)  & $\frac{-3}{1\binom{1}{1}}=-3$ & $\frac{5}{1\binom{2}{1}}=2.5$ & $\frac{-2}{1\binom{3}{1}}=-0.667$ & -1.167 \\
$\phi_2$ (Shapley for $x_2$)  & 0 & 0 & $\frac{2}{2\binom{3}{2}}=0.333$ & 0.333 \\
$\phi_3$ (Shapley for $x_3$)  & 0 & $\frac{-5}{1\binom{2}{1}}=-2.5$ & $\frac{2}{2\binom{3}{2}}=0.333$ & -2.167 \\
\end{tabular}}
\caption{\label{shapley_calc_example}Detailed Shapley values calculation on $\psi$ (as defined in Table \ref{banzhaf_calc_example}). Observe that the efficiency property holds as the sum of the Shapley values is $-3$ and: $\psi(x_1=x_2=x_3=1) - \psi(x_1=x_2=x_3=0) = -3$  }
\end{table}

In the following subsections, we will prove that these formulas correctly compute the Shapley and Banzhaf values, and demonstrate how to use them to calculate the Shapley/Banzhaf value for all variables in linear time.

\subsection{Correctness Proofs}
\label{sec:shap_banzhaf_proof}

We will mark the linear-time formulas as ``simplified'' and the exponential time formulas as ``original''. Mark Formula~\ref{PB_shap} as $\so$, Formula~\ref{PB_banzhaf} as $\bo$, Formula~\ref{shap_simplified_formula} as $\shapsimp$, and Formula~\ref{banzhaf_simplified_formula} as $\bs$.

\begin{theorem}\label{simplified_banzhaf}
For any player $i$, the simplified Banzhaf formula $\bs$ is equal to the original formula $\bo$ when the input is in WDNF.
\end{theorem}
\begin{proof}
Due to the linearity property, Banzhaf values can be calculated by computing them for each cube independently and summing the results. Given a weighted cube $c_k$ and a player $i$ we will show $\bs(c_k) = \bo(c_k)$. 

We will consider three scenarios: 
\begin{enumerate}
    \item $i \notin S_k$: From null player property: $\bo(c_k)=0$, and due to the third case in Formula~\ref{banzhaf_simplified_formula}: $\bs(c_k) = 0$.
    \item $i \in S^+_k$: As the literal $x_i$ is included in the cube, any assignment with $x_i=0$ will leave the cube unsatisfied: 
    
    $\forall x_{-i}: c_k(x_{-i},\:  x_i=0)=0$. 
    
    This property allows us to simplify $\bo$ (Formula~\ref{PB_banzhaf}) by directly substituting $c_k(x_{-i},\:  x_i=0)=0$: 

    \begin{equation*}
        \bo(c_k) = \frac{1}{2^{h-1}}\sum\limits_{x_{-i} \in \{0,1\}^{h-1}}c_k(x_{-i},\:  x_i=1) - 0
    \end{equation*}
    
    Next, we will compute the summation by counting the number of satisfying assignments. The cube is satisfiable because we have already removed all cubes $c_j$ s.t. $|S^+(c_j)\cap S^-(c_j)| \ge 1$. There are $h-1$ variables (ignoring $x_i$ as we already assigned it to 1, see Defs.~\ref{def:WDNF} and~\ref{def:WCNF}), $|S_k|-1$ of them appear in $c_k$. To satisfy $c_k$, we must correctly assign all its variables. The remaining variables can have any assignment. Therefore, $c_k(x_{-i},\:  x_i=1)$ has $2^{h-1-(|S_k|-1)}$ satisfying assignments, each contributes a weight of $w_k$. 
    \[
    \bo(c_k) = \frac{1}{2^{h-1}} \cdot w_k \cdot 2^{h-1-(|S_k|-1)} =  \frac{w_k}{ 2^{|S_k|-1}}
    \]
    \item $i \in S^-_k$: Now any assignment with $x_i=1$ will leave the cube unsatisfied. We will use the same technique as above, simplifying the original formula and counting the number of satisfying assignments. \newline
    \[
    \bo(c_k) = \frac{1}{2^{h-1}}\sum\limits_{x_{-i} \in \{0,1\}^{h-1}}0-c_k(x_{-i},\:  x_i=0) = 
    \]
    \[
    \frac{1}{2^{h-1}} \cdot (-w_k) \cdot 2^{h-1-(|S_k|-1)} = \frac{-w_k}{ 2^{|S_k|-1}}
    \]
\end{enumerate}

\end{proof}

\begin{theorem}\label{simplified_shapely}
For any player $i$, the simplified Shapley formula $\shapsimp$ is equal to the original formula $\so$ when the input is in WDNF.
\end{theorem}

\begin{proof}
Using the linearity property, it is sufficient to show that for a weighted cube $c_k$ and a player $i$, we have $\shapsimp(c_k) = \so(c_k)$.

This proof utilizes the Shapley value's \emph{null player out (NPO)} property. The \emph{null player} property guarantees that players that have zero contribution will get Shapley value of zero. \emph{Null player out} considers the case when null players are removed completely from the game. 

By combining the \emph{null player} property with the \emph{efficiency} property, we can deduce that removing null players from the game does not affect the sum of the Shapley values of the remaining players. The question then becomes whether the removal of a null player can redistribute the Shapley values, maintaining their sum but altering how they are distributed among the remaining players. The \emph{null player out} (NPO) property assures that this will not occur. The NPO property guarantees that the Shapley values of the other players remain unchanged when null players are removed or added. For a simple proof of this property, see Lemma 2 in \cite{understandinginterventionaltreeshap}, and for an in-depth review, refer to \cite{NPO}.

In our proof, we will remove all null players, observe that after removal, only one assignment satisfies $c_k$, and use this observation to derive $\shapsimp$ from $\so$. We consider three separate cases:

\begin{enumerate}
    \item $i \notin S_k$: From null player property and applying the $\shapsimp$ formula: $\so(c_k)=0=\shapsimp(c_k)$.
    \item $i \in S^+_k$: As the literal $x_i$ is included in the cube, any assignment with $x_i=0$ will leave the cube unsatisfied. We can use this to simplify the formula $\so$ by directly substituting $c_k(x_{-i},\:  x_i=0)=0$:  

    \begin{equation}
    \begin{split}
    \so(w_k \cdot c_k) = \\
    \sum\limits_{x_{-i} \in \{0,1\}^{h-1}}
     \frac{(\sum\limits_{x_j \in x_{-i}}x_j)!(h - (\sum\limits_{x_j \in x_{-i}}x_j) - 1)!}{h!} \cdot \Delta c, \\
    \Delta c = (w_k \cdot c_k(x_{-i},\:  x_i=1)- 0)
    \label{shap_on_cube_i_positive}
    \end{split}
    \end{equation}

    The cube is satisfiable as we already removed all cubes where $|S^+(c_j)\cap S^-(c_j)| \ge 1$. Since all variables $x_j \notin S_k$ are null players, the NPO property allows us to remove them without changing $\so(w_k \cdot c_k)$. After removing them, we are left with only a single assignment that satisfies the cube. We will mark this assignment, excluding $x_i$, as $a$. 
    The assignment is:

    \begin{equation*}
    a = 
    \begin{cases}
    1 & \text{if } x_j \in S^+_k \\
    0 & \text{if } x_j \in S^-_k
    \end{cases}
    \;\; | \;\; \forall x_j \in S_k, j \neq i
    \}
    \end{equation*}

    Observe that $|a| = |S_k| - 1$ as $a$ includes all the cube's variables excluding $x_i$. We will mark:
    \[
    a = \{a_1, a_2, \dots , a_{|S_k|-1}\}.
    \]

   Considering this, we can simplify Formula~\ref{shap_on_cube_i_positive} above. Instead of going over all assignments, we only need to consider the assignment $a \cup \{x_i = 1\}$, as this is the only one that satisfies the cube:

    \begin{equation*}
    \begin{split}
    \so(c_k) = \\
    \frac{(\sum_{j=1}^{|S_k|-1}a_j)!(|S_k| - (\sum_{j=1}^{|S_k|-1}a_j) - 1)!}{|S_k|!}w_k \stackrel{(a)}= \\
    \frac{(|S^+_k|-1)!(|S_k| - |S^+_k|)!}{|S_k|!}w_k \stackrel{(b)}= \\
    \frac{w_k}{|S^+_k|\binom{|S_k|}{|S^+_k|}} \stackrel{(c)}=     \shapsimp(c_k)
    \end{split}
    \end{equation*}
    \begin{enumerate}
        \item Equation~\ref{shap_on_cube_i_positive} when only considering the single satisfying assignment: 
        \begin{enumerate}
            \item We remove the $\sum_{x_{-i} \in \{0,1\}^{h-1}}$ as there is only one satisfying assignment.
            \item Replace $h$ with $|S_k|$ (the size of $a \cup \{x_i=1\}$, as we removed all null players).
            \item Replace $\sum_{x_j \in x_{-i}}x_j$ with $\sum_{j=1}^{|S_k|-1}a_j$ as we consider the case where $a=\{a_1, \dots , a_{|S_k|-1}\} =x_{-i}$.
            \item Replace $c_k(x_{-i},\:  x_i=1)$ with $w_k$ as the assignment satisfies the cube.
        \end{enumerate}
       \item $\sum_{j=1}^{|S_k|-1}a_j = |S^+_k| - 1$ as $a$ includes $|S^+_k| - 1$ variables that are set to 1. See the definition of $a$ above.

        \item Simple arithmetic.
    \end{enumerate}
    
   \item $i \in S^-_k$: Now, any assignment with $x_i = 1$ will leave the cube unsatisfied. We will use the same technique as above, simplifying the original formula, removing null players, observing that after the removal there is only one satisfying assignment, and simplifying the formula even further using this observation:

    \resizebox{\linewidth}{!}{$
    \begin{aligned}
    \so(c_k) = \\
    \sum\limits_{x_{-i} \in \{0,1\}^{h-1}} \frac{(\sum\limits_{x_j \in x_{-i}}x_j)!(h - (\sum\limits_{x_j \in x_{-i}}x_j) - 1)!}{h!}(0 - c_k(x_{-i},\:  x_i=0)) &\stackrel{(a)}= \\
    \frac{(\sum_{j=1}^{|S_k|-1}a_j)!(|S_k| - (\sum_{j=1}^{|S_k|-1}a_j) - 1)!}{|S_k|!}(-w_k) &\stackrel{(b)}= \\
    \frac{(|S^+_k|)!(|S_k| - |S^+_k| - 1)!}{|S_k|!}(-w_k) &\stackrel{(c)}= \\
    \frac{(|S_k| - |S^-_k|)!(|S^-_k| - 1)!}{|S_k|!}(-w_k) &\stackrel{(d)}= \\
    \frac{-w_k}{|S^-_k|\binom{|S_k|}{|S^-_k|}} &\stackrel{(e)}= \\ \shapsimp(c_k)
    \end{aligned}
    $}
    \begin{enumerate}
        \item As $x_i\in S^-_k$, any assignment with $x_i=1$ will leave the cube unsatisfied. 
        \item Only considering the single satisfying assignment (with $a$ defined as in the $x_i \in S^+_k$ case).
        \item Now when $i \notin S^+_k$, the sum $(\sum_{j=1}^{|S_k|-1}a_j)$ for the satisfying assignment is $|S^+_k|$ and not $|S^+_k| - 1$.
        \item We apply the identity $|S_k| = |S^-_k| + |S^+_k|$, since all cubes with $|S^+_k\cap S^-_k| \ge 1 $ have been excluded.
        \item Simple arithmetic.
    \end{enumerate}
\end{enumerate}

\end{proof}

A direct corollary of Theorems~\ref{simplified_banzhaf} and~\ref{simplified_shapely} is that the simplified formulas for Shapley and Banzhaf values are robust w.r.t WDNF representations: equivalent WDNF will get the same Shapley and Banzhaf values.

\begin{corollary}\label{robust_DNF}
The simplified formulas are robust. For two equal WDNF formulas $\forall x : F_1(x)=F_2(x)$, even if they are represented by different cubes, $\bs(F_1)=\bs(F_2)$ and $\shapsimp(F_1)=\shapsimp(F_2)$ for any player $i$.
\end{corollary}

\begin{proof}
For any player i, $\so$ and $\bo$ are robust as they are functions of assignments: they use $F$ by calling it ($F(x)$) and not by considering its cube structure. Using Theorems~\ref{simplified_banzhaf} and ~\ref{simplified_shapely}:

\[
\bs(F_1) = \bo(F_1) = \bo(F_2) = \bs(F_2)
\]
\[
\shapsimp(F_1) = \so(F_1) = \so(F_2) = \shapsimp(F_2)
\]

\end{proof}

While the proof of Corollary~\ref{robust_DNF} is straightforward, based on Theorems~\ref{simplified_banzhaf} and ~\ref{simplified_shapely}, the statement itself is not trivial, as the simplified formulas only consider the structure of the WDNF. For example, the WDNF:
\[
5(x_1) - 5(x_3) + 3(\neg x_1 \land \neg x_3) + 10(\neg x_1 \land x_3) \]\[ - 2(\neg x_1 \land \neg x_2 \land x_3)
\]
is equal to the WDNF in Tables~\ref{banzhaf_calc_example} and~\ref{shapley_calc_example}. Calculating its Banzhaf and Shapley values using the simplified formulas will return the same values as in the tables. More trivial functions over cubes, like the simple weight difference of positive versus negated literals, presented in Formula~\ref{formula_make_break}, are not robust.

\begin{equation} \label{formula_make_break}
\Delta W_i =\sum_{k=1, i \in S^+_k}^m w_k - \sum_{k=1, i \in S^-_k}^m w_k
\end{equation}
 
Next, we show that $\bs$ and $\shapsimp$ also apply to WCNF (see Def.~\ref{def:WCNF}).

\begin{theorem}\label{simplified_holds_for_cnf}
For any player $i$ and a Weighted Conjunctive Normal Form (WCNF) function $F$ the simplified formulas are equal to the original ones: $\shapsimp(F) = \so(F)$, $\bs(F) = \bo(F)$.
\end{theorem}

\begin{proof}
One can carefully apply the proofs of Theorems~\ref{simplified_banzhaf} and~\ref{simplified_shapely} above to WCNF, changing them slightly to consider the new form. We present another approach that utilizes the robustness property.

Given a WCNF formula, we will transform it to WDNF, run $\bs$ and $\shapsimp$ on this WDNF, and show that it is equivalent to running them directly on the WCNF. As $\bs$, $\shapsimp$, $\bo$, and $\so$ are robust, transforming the WCNF formula to an equivalent WDNF formula does not change the Shapley/Banzhaf values.

\[
w\psi = w - w(\neg\psi)
\]
For example:
\[
3(x_1 \lor \neg x_3) = 3 - 3(\neg x_1 \land x_3)
\]

Applying the transformation on $F$ we get the WDNF formula $F' = \sum_{k=1}^m w_k - w_k(\neg c_k)$. 
Calculating Shapley/Banzhaf on $F'$:
\begin{enumerate}
    \item $w_k - w_k(\neg c_k)$ has a number $w_k$ and a weighted cube. We can treat the number as a weighted cube with zero literals and observe that constant numbers have no effect on Shapley/Banzhaf values.
    \item The original weight of the clause was multiplied by $-1$.
    \item Using de-Morgan's laws:
    \[
    |S^+(c_k)| = |S^-(\neg c_k)|, \quad |S^-(c_k)| = |S^+(\neg c_k)|
    \]
\end{enumerate}
When applying 2 and 3 to $\bs$ and $\shapsimp$, both changes cancel out and we are left with the same equations.
\end{proof}

A direct corollary of Theorem~\ref{simplified_holds_for_cnf} is the robustness of $\bs$ and $\shapsimp$ also w.r.t WCNF formulas.

\begin{corollary}
\label{theorem_robust_for_WDNF_WCNF}
For any player $i$, both $\bs$ and $\shapsimp$ are robust for both WDNF and WCNF.
\end{corollary}

\begin{proof}
Identical to Corollary~\ref{robust_DNF} proof. The robustness holds even if one formula is a WCNF and the other is a WDNF.
\end{proof}

To summarize all of the above: 

\begin{theorem}\label{summary}
For any player $i$, $\bs = \bo$ and $\shapsimp = \so$ for both WDNF and WCNF. These simplified formulas are robust.
\end{theorem}


\subsection{Complexity Overview}
\label{PB_calc_complexity}

First, we demonstrate how to use $\bs$ and $\shapsimp$ to calculate the values of all the variables in the WDNF/WCNF formula in linear-time.

\begin{proposition}\label{linear_time_calculation}
\textbf{Linear Complexity}: $\shapsimp$ and $\bs$ can calculate the Shapley/Banzhaf values of all variables in linear-time with respect to the length of the WDNF/WCNF formula. In a single linear-time calculation, the Shapley/Banzhaf values of all players can be obtained.
\end{proposition}

\begin{proof}
The computation of Shapley and Banzhaf values begins by initializing a dictionary that maps each variable to zero. Throughout the process, this dictionary is updated to store the current Shapley or Banzhaf value of each variable. The algorithm then traverses the formula, visiting each cube once, and incrementally updates the mapping with the partial values accumulated up to that cube. For each cube, it iterates over its variables and computes their contributions using the simplified formulas.

This process is illustrated in Tables \ref{banzhaf_calc_example} and \ref{shapley_calc_example}. Conceptually, the algorithm “fills” these tables column by column - updating only the cells corresponding to variables present in the current cube, while the remaining entries remain zero.

For a cube $c_k$, with $|S_k|$ variables, a naive implementation might require $O(|S_k|^2)$, as calculating $\frac{1}{ 2^{|S_k|-1}}$, $\frac{1}{|S^-_k|\binom{|S_k|}{|S^-_k|}}$, and $\frac{1}{|S^+_k|\binom{|S_k|}{|S^+_k|}}$ takes $O(|S_k|)$ time for each of the $|S_k|$ variables. 

However, we can optimize this by computing these fractions once per cube and reusing them. This approach requires $O(|S_k|)$ operations per cube for the fraction calculations and another $O(|S_k|)$ for the Shapley/Banzhaf value computations. Consequently, the overall complexity becomes linear with respect to the WDNF/WCNF length.
\end{proof}

Let us formally define Conjunctive Normal Form (CNF) and Disjunctive Normal Form (DNF):

\begin{definition}[Conjunctive Normal Form (CNF)]
\label{def:cnf}
    Conjunctive Normal Form (CNF) is a set of clauses, where the function returns 1 if all clauses are satisfied and 0 otherwise.
\end{definition}

\begin{definition}[Disjunctive Normal Form (DNF)]
\label{def:dnf}
    Disjunctive Normal Form (DNF) is a set of cubes, where the function returns 1 if at least one cube is satisfied and 0 otherwise.
\end{definition}

While Banzhaf and Shapley values on WCNF and WDNF can be calculated in linear time, calculating Shapley values on CNF and DNF formulas is \#P-Hard. 

Model counting is the task of counting the number of satisfying assignments; for CNF and DNF formulas, this is known to be \#P-Hard~\cite{counting_dnf_is_hard}. A polynomial reduction between Shapley values computation and the model counting problem, demonstrating that computing Shapley values on CNF and DNF formulas is \#P-Hard, is presented in~\cite{arenas2021tractabilityshapscorebasedexplanationsdeterministic}. Below, we provide a simple reduction between Banzhaf value calculation and model counting on CNF and DNF formulas, proving that Banzhaf value calculation on these formulas is also \#P-Hard.

\begin{theorem}\label{banzhaf_is_hard}
Banzhaf calculation on CNF and DNF is \#P-Hard.
\end{theorem}
\begin{proof}
Given a CNF/DNF formula $\psi$ with $n$ variables, we build the formula $\psi'$ by adding a variable $x_{n+1}$ and a cube/clause $(x_{n+1})$ — a cube/clause of size 1 that includes only the positive literal $x_{n+1}$.
\begin{itemize}
    \item In CNF, when $x_{n+1}=0$, the new formula $\psi'$ is never satisfied, as $(x_{n+1})$ is false. When $x_{n+1}=1$, the new formula is satisfied if and only if the old formula is satisfied:
    
    \[
    \beta_{n+1}(\psi') = \mathbb{E}[\psi'(x)|x_{n+1}=1] - \mathbb{E}[\psi'(x)|x_{n+1}=0] =
    \]
    \[
    \mathbb{E}[\psi(x)] = \frac{1}{2^{|N|}}\sum_{x \in \{0,1\}^n} \psi(x)
    \]
    
    \item In DNF, when $x_{n+1}=0$, the new formula $\psi'$ is satisfied if and only if the old formula is satisfied. When $x_{n+1}=1$, the new formula is always satisfied, as $(x_{n+1})$ is true:
    
    \[
    \beta_{n+1}(\psi') = \mathbb{E}[\psi'(x)|x_{n+1}=1] - \mathbb{E}[\psi'(x)|x_{n+1}=0] = 
    \]
    \[
    1 - \mathbb{E}[\psi(x)] = 1 - \frac{1}{2^{|N|}}\sum_{x \in \{0,1\}^n} \psi(x)
    \]
    
\end{itemize}
The model counting of $\psi$ is $\sum_{x \in \{0,1\}^n} \psi(x)$. This can be easily extracted from $\beta_{n+1}$ in both cases above. We have shown a simple reduction between Banzhaf value computation and model counting, thus proving that Banzhaf value computation is \#P-Hard.
\end{proof}

See a summary of the complexity of Shapley and Banzhaf values computation in Table~\ref{complexity_on_dnf_cnf}.

\begin{table}[h]
\resizebox{\columnwidth}{!}{%
\begin{tabular}{l|l|l}
Formula type & Shapley Complexity & Banzhaf Complexity \\\hline
CNF & \#P-Hard & \#P-Hard \\
DNF & \#P-Hard & \#P-Hard \\
Weighted CNF & Linear time & Linear time \\
Weighted DNF & Linear time & Linear time \\
\end{tabular}}
\caption{\label{complexity_on_dnf_cnf}Shapley and Banzhaf values calculation complexity on the different Boolean formulas. }
\end{table}

\section{Shapley and Banzhaf Interaction Values on WDNF and WCNF - With Correctness Proofs}
\label{sec:interaction_values}

Shapley interaction values are defined in Def.~\ref{definition_shap_interaction_values}, and their efficient computation for WDNF/WCNF formulas is presented in Table~\ref{table_shap_iv}.
Banzhaf interaction values are defined in Def.~\ref{banzhaf_iv_def} below, and can be efficiently computed on WDNF/WCNF formulas using Formula~\ref{banzhaf_simplified_formula_interactions}.

\begin{definition}[Banzhaf Interaction Values]  
\label{banzhaf_iv_def}  
Banzhaf Interaction values examine the difference in expectations when features $i$ and $j$ are both missing or participating, versus when only one participates \cite{interaction_values}:  

\begin{multline}  
\beta_{i,j \; i \neq j}(M) = \mathbb{E} \left[M(S) \mid i,j \in S\right] + \mathbb{E} \left[M(S) \mid i,j \notin S\right] \\
- \mathbb{E} \left[M(S) \mid i \in S \land j \notin S\right] - \mathbb{E} \left[M(S) \mid i \notin S \land j \in S\right]  
\end{multline}  
\end{definition}

\paragraph{Interaction values on WCNF:} To compute Shapley and Banzhaf interaction values on WCNF, use Table~\ref{table_shap_iv} and Formula~\ref{banzhaf_simplified_formula_interactions}, and multiply their outputs by $-1$. The correctness of this approach follows from proofs analogous to those of Table~\ref{table_shap_iv} (Theorem~\ref{shapley_interaction_table_is_correct}) and Formula~\ref{banzhaf_simplified_formula_interactions} (Theorem~\ref{banzhaf_interaction_table_is_correct}).

\paragraph{Complexity of the Efficient Approach:} Mark $L$ as the number of cubes and $D$ as the number of literals in each cube. While the calculation of Shapley and Banzhaf values has linear complexity $O(LD)$, computing all interaction values requires examining all pairs of literals in each cube, resulting in a complexity of $O(LD^2)$.

\subsection{Correctness Proofs}

\begin{lemma} \label{shap_iv_linearity_proof}
[Linearity Property of Shapley Interaction Values (folklore)]  
Given a WDNF/WCNF formula $F = \sum_{k=1}^m w_k \cdot c_k$, for any players $i, j$ such that $i \neq j$, the following holds:
\[
\phi_{i,j\; i \neq j}(F) = \sum_{k=1}^m \phi_{i,j\; i \neq j}(c_k)
\]
\end{lemma}

\begin{proof}
The linearity property holds for Shapley interaction values because they are computed as a linear combination of standard Shapley values:

\[\phi_{i,j\;i \neq j}(F)=\phi_{j|i=1}(F)-\phi_{j|i=0}(F) = \]\[
\sum_{k=1}^m\phi_{j|i=1}(c_k) - \sum_{k=1}^m\phi_{j|i=0}(c_k) = \]\[
\sum_{k=1}^m [\phi_{j|i=1}(c_k)-\phi_{j|i=0}(c_k)] = \sum_{k=1}^m \phi_{i,j\;i \neq j}(c_k)\]

\end{proof}

\begin{theorem}\label{shapley_interaction_table_is_correct}
The formulas in Table~\ref{table_shap_iv} correctly calculate Shapley interaction values (as defined in Def.~\ref{definition_shap_interaction_values}) on WDNF formulas.
\end{theorem}

\begin{proof}
Given a WDNF formula $F$, from the linearity property (Lemma~\ref{shap_iv_linearity_proof}), it is enough to show that the formula holds for each cube $c_k$ separately. Let us consider all possible cases and summarize them in Table~\ref{table_shapley_iv_proof}:
\begin{enumerate}
    \item $i \notin S_k$: Setting $i$ to 0 or 1 will never change the satisfied weight, and $\phi_{j|i=1} = \phi_{j|i=0} = \phi_j$.
    \item $j \notin S_k$: By the null player property, $\phi_{j|i=1} = \phi_{j|i=0} = 0$.
    \item $i,j \in S^+_k$: When $i = 0$, the cube $c_k$ is always falsified, so $\phi_{j|i=0} = 0$ (the Shapley values of the variables participating in a constant function are always zero, as the variables have no effect on the function output. We used this fact in the proof of Theorem~\ref{simplified_holds_for_cnf}). When we set $i = 1$, we are left with a cube $c_k'$ such that $|S(c_k')| = |S(c_k)| - 1$, $|S^+(c_k')| = |S^+(c_k)| - 1$, and $|S^-(c_k')| = |S^-(c_k)|$. Let $V = \{x_1, x_2, \ldots, x_h\}$ be the Boolean variables of $F$ (see Def.~\ref{def:WDNF}). Using Def.~\ref{definition_shap_interaction_values} and assignment in Formula~\ref{shap_simplified_formula}, we get: 
    
    \[
    \phi_{i,j\;i \neq j}(c_k) = \phi_j(c_k(V|x_i=1)) - \phi_j(c_k(V|x_i=0)) = 
    \]\[
    \frac{w_k}{(|S^+(c_k)|-1)\binom{|S(c_k)| - 1}{|S^+(c_k)|-1}} - 0
    \]
    \item $i,j \in S^-_k$: Symmetrically, when $i = 1$, the cube is always falsified, and $\phi_{j|i=1} = 0$. When we set $i = 0$, we are left with a cube $c_k'$ such that $|S(c_k')| = |S(c_k)| - 1$, $|S^+(c_k')| = |S^+(c_k)|$, and $|S^-(c_k')| = |S^-(c_k)| - 1$. Using Def.~\ref{definition_shap_interaction_values} and assignment in Formula~\ref{shap_simplified_formula}, we get:
    
    \[
    \phi_{i,j\;i \neq j}(c_k) = 0 - \frac{-w_k}{(|S^-(c_k)|-1)\binom{|S(c_k)| -1}{|S^-(c_k)|-1}}
    \]
    \item $i \in S^+_k$ and $j \in S^-_k$: We know $\phi_{j|i=0} = 0$, and assigning $i = 1$ results in $c_k'$ such that $|S(c_k')| = |S(c_k)| - 1$ and $|S^-(c_k')| = |S^-(c_k)|$: 
    
    \[
    \phi_{i,j\;i \neq j}(c_k) = \frac{-w_k}{(|S^-(c_k)|)\binom{|S(c_k)| - 1}{|S^-(c_k)|}} - 0
    \]
    \item $i \in S^-_k$ and $j \in S^+_k$: Symmetrically: 
    
    \[
    \phi_{i,j\;i \neq j}(c_k) = 0 - \frac{w_k}{(|S^+(c_k)|)\binom{|S(c_k)| - 1}{|S^+(c_k)|}}
    \]
\end{enumerate}

We have shown the formulas hold for all cases.  Table~\ref{table_shapley_iv_proof} summarizes all cases.
\end{proof}

\begin{table}[h!]
\centering
\renewcommand{\arraystretch}{1.5}
\caption{\label{table_shapley_iv_proof}Summarize the Shapley interaction values proof. For each case we calculate both $\phi_{j|i=1}$ and  $\phi_{j|i=0}$ to find the interaction value. We denote $S=S(c_k) = S_k$, $S^+=S^+(c_k)=S^+_k$, $S^-=S^-(c_k)=S^-_k$, $w=w_k$.}
\resizebox{\columnwidth}{!}{
\begin{tabular}{l|l|l|l}
Case & $\phi_{j|i=1}$ & $\phi_{j|i=0}$ & $\phi_{i,j\;i \neq j}$ \\\hline
$(i \in S^+)\land (j \in S^-)$ & $\frac{-w}{(|S^-|)\binom{|S| - 1}{|S^-|}}$ & 0 & $\frac{-w}{(|S^-|)\binom{|S| - 1}{|S^-|}}$ \\
$(i \in S^-)\land (j \in S^+)$ & 0 & $\frac{w}{(|S^+|)\binom{|S| - 1}{|S^+|}}$ & $\frac{-w}{(|S^+|)\binom{|S| - 1}{|S^+|}}$ \\
$i,j \in S^+$ & $\frac{w}{(|S^+|-1)\binom{|S| - 1}{|S^+|-1}}$ & 0 & $\frac{w}{(|S^+|-1)\binom{|S| - 1}{|S^+|-1}}$ \\
$i,j \in S^-$ & 0 & $\frac{-w}{(|S^-|-1)\binom{|S| - 1}{|S^-|-1}}$ & $\frac{w}{(|S^-|-1)\binom{|S| - 1}{|S^-|-1}}$ \\
$(i \notin S)\land (j \in S^+)$ & $\frac{w}{(|S^+|)\binom{|S|}{|S^+|}}$ & $\frac{w}{(|S^+|)\binom{|S|}{|S^+|}}$ & 0 \\
$(i \notin S)\land (j \in S^-)$ & $\frac{-w}{(|S^-|)\binom{|S|}{|S^-|}}$ & $\frac{-w}{(|S^-|)\binom{|S|}{|S^-|}}$ & 0 \\
$j \notin S$ & 0 & 0 & 0 \\
\end{tabular}
}
\end{table}

\begin{lemma}
\label{banzhaf_iv_def_2}
Banzhaf interaction values can also be defined as follows:
\begin{equation}
    \beta_{i,j\;i \neq j} = \beta_{j|i=1} - \beta_{j|i=0}
\end{equation}
\end{lemma}

\begin{proof} 
Combining Def.~\ref{banzhaf_iv_def} with Formula~\ref{banzhaf_expectations_def}:
\begin{multline*}
\beta_{i,j\;i \neq j}(M) = \\ \mathbb{E} [M(S)|i,j \in S] + \mathbb{E} [M(S)|i,j \notin S] \\
- \mathbb{E} [M(S)|i \in S \land j \notin S] - \mathbb{E} [M(S)|i \notin S \land j \in S] = \\
(\mathbb{E} [M(S)|j \in S \land i \in S] - \mathbb{E} [M(S)|j \notin S \land i \in S]) - \\
(\mathbb{E} [M(S)|j \in S \land i \notin S] - \mathbb{E} [M(S)|j \notin S \land i \notin S]) = \\
\beta_{j|i=1}(M) - \beta_{j|i=0}(M)
\end{multline*}
\end{proof}

\begin{theorem}\label{banzhaf_interaction_table_is_correct}
Formula~\ref{banzhaf_simplified_formula_interactions} correctly calculates Banzhaf interaction values on WDNF formulas.
\end{theorem}

\begin{proof}
The linearity property holds for the Banzhaf interaction values due to the linearity of expectations and Definition~\ref{banzhaf_iv_def}. Utilizing the definition presented in Lemma~\ref{banzhaf_iv_def_2}, we can apply the proof structure of Shapley interaction values here and derive Formula~\ref{banzhaf_simplified_formula_interactions}. Table~\ref{table_banzhaf_iv_proof} summarizes the cases in the proof. 
\end{proof}

\begin{table}[h]
\centering
\renewcommand{\arraystretch}{1.5}
\caption{\label{table_banzhaf_iv_proof}Summary of the Banzhaf interaction values proof. For each case, we calculate both $\beta_{j|i=1}$ and $\beta_{j|i=0}$ to find the interaction value. We denote $S=S(c_k) = S_k$, $S^+=S^+(c_k)=S^+_k$, $S^-=S^-(c_k)=S^-_k$, $w=w_k$.}
\begin{tabular}{l|l|l|l}
Case & $\beta_{j|i=1}$ & $\beta_{j|i=0}$ & $\beta_{i,j\;i \neq j}$ \\\hline
$(i \in S^+)\land (j \in S^-)$ & $\frac{-w}{2^{|S|-2}}$ & 0 & $\frac{-w}{2^{|S|-2}}$ \\
$(i \in S^-)\land (j \in S^+)$ & 0 & $\frac{w}{2^{|S|-2}}$ & $\frac{-w}{2^{|S|-2}}$ \\
$i,j \in S^+$ & $\frac{w}{2^{|S|-2}}$ & 0 & $\frac{w}{2^{|S|-2}}$ \\
$i,j \in S^-$ & 0 & $\frac{-w}{2^{|S|-2}}$ & $\frac{w}{2^{|S|-2}}$ \\
$(i \notin S)\land (j \in S^+)$ & $\frac{w}{2^{|S|-1}}$ & $\frac{w}{2^{|S|-1}}$ & 0 \\
$(i \notin S)\land (j \in S^-)$ & $\frac{-w}{2^{|S|-1}}$ & $\frac{-w}{2^{|S|-1}}$ & 0 \\
$j \notin S$ & 0 & 0 & 0 \\
\end{tabular}
\end{table}

\section{Complexity Analysis}
\label{sec:complexity_analysis}

This section presents a complexity analysis of our algorithms. We begin by analyzing the complexity of \emph{MapPatternsToCube} and \emph{CaclDecisionPatterns}, followed by an analysis of our main algorithm \algname. The input for our algorithms is:
\begin{itemize}
    \item A decision tree ensemble with $T$ trees, each tree with depth $D$, $K$ nodes and $L$ leaves. In the complexity analysis we assume the trees are more or less balanced, i.e. $O(K) = O(2^D) = O(L)$.
    \item Consumer data $C$ with $n$ rows
    \item Background data $B$ with $m$ rows
    \item A function $v$ that takes a cube and returns a mapping from feature subsets (of size one or more) to real numbers. This function can compute Shapley/Banzhaf values, interaction values and any other metric that satisfies the linearity property.
\end{itemize}

\paragraph{\textit{CalcDecisionPatterns} complexity:} The algorithm's input is a single decision tree and the consumer data $C$. The algorithm performs a Breadth-First Search (BFS), applying the split function of each node to all consumers. This results in a time complexity of $O(nK) = O(nL)$. The output size is also $O(nL)$, indicating that the algorithm is asymptotically optimal.

\paragraph{\textit{MapPatternsToCube} complexity:} The algorithm's input is the list of features along a root-to-leaf path, which has length $D$. The algorithm iterates over this list, tripling the number of entries in the dictionary $d$ at each step. As a result, both the runtime and output size are $O(3^D)$, indicating that the algorithm is asymptotically optimal.

\paragraph{Function $v$ complexity:}  \algname complexity depends on the function $v$. In this section we use amortized complexity analysis and consider the average-case complexity and output size of $v$ over all possible cubes defined on $\{x_1, \dots , x_D\}$. 

The average case complexity for cubes of length up to $D$ is defined as the total worst-case complexity of running $v$ on all possible cubes over $\{x_1, \dots , x_D\}$ (or any subset of these variables), divided by the number of cubes, $3^D$. Similarly the average case output size is the total worst-case output size on these cubes divided by $3^D$. Here, the $\{x_1, \dots , x_D\}$ variables represents features along a root-to-leaf path. We consider the worst-case where all paths in the tree include unique features (no path with repeating features). We denote the average output size by $O(v_o^{avg})$ and the average complexity by $O(v_c^{avg})$. Running $v$ on all cubes then costs $O(3^D v_c^{avg})$ and produces $O(3^D v_o^{avg})$ outputs. Clearly, $O(v_o^{avg}) \le O(v_c^{avg})$.

For Shapley/Banzhaf values, $O(v_c^{avg}) = O(v_o^{avg}) = O(D)$, as for each cube the function $v$ computes and returns one value per each variable. For Shapley/Banzhaf interaction values, $O(v_c^{avg}) = O(v_o^{avg}) = O(D^2)$, as for each cube the function $v$ computes and returns one value per each pair of variables. 

For both Shapley/Banzhaf values and their interaction values, the average-case complexity and output size are asymptotically the same as the complexity and output size on the worst case cube/clause. Thus, using average-case instead of worst-case measures does not change the overall complexity for these values. We still provide the tightest bounds we can, hoping they will be useful for future research.  

Another property of $v$ needed in the analysis is $O(v_s)$, the number of unique feature subsets returned by all possible cubes over $\{x_1, \dots , x_D\}$. For Shapley/Banzhaf values $O(v_s) = D$, since $v$ returns all subsets of size one. For Shapley/Banzhaf interaction values $O(v_s) = D^2$, since $v$ returns all subsets of size two.

\paragraph{\algname complexity:} The algorithm has several steps: computing $f$, $M$, $s$, and finally the Shapley/Banzhaf values. We analyze the complexity of each step individually.

\begin{enumerate}
    \item \textbf{Computing $\mathbf{f}$}: 
    
    \underline{In Background}:
    Calling \emph{CalcDecisionPatterns(T,B)} takes $O(mL)$ per tree and returns an output of size $O(mL)$. The \textit{value\_counts} function and \textit{bincount} function, a more efficient Numpy alternative we used in practice, processes this output in linear time, resulting in an overall complexity of $O(mTL)$.
    
    \underline{In Path-Dependent}:
    Formula~\ref{eq:path_dependent_estimation} can be computed for all leaves in a single traversal of the tree, requiring $O(L2^D)=O(L^2)$ time per tree and $O(TL^2)$ overall.
    
    \item \textbf{Computing $\mathbf{M}$}:  
    For each leaf in each tree we iterate over all $O(3^D)$ cubes generated by \emph{MapPatternsToCube}, apply $v$ to each cube, and iterate over its output. These cubes are all the possible cubes over the $D$ variables that represents the features along the root-to-leaf path. Thus, the complexity of the step is $O(TL3^D v_c^{avg})$. Note: for this analysis to be correct we assume $1 \le v_c^{avg}$, even if most cubes are skipped the average complexity can not be less than $O(1)$.
    
    This step constructs $O(TL v_s)$ matrices, and these matrices together contain $O(TL3^D v_o^{avg})$ non-zero entries. Each matrix is sparse: its size is $4^D$ (since both $p_c$ and $p_b$ can range between $0$ and $2^D$), but it holds at most $3^D$ non-zero entries.

    \item \textbf{Computing $\mathbf{s}$}: 
    The bottleneck here is the matrix-vector multiplication $\mathbf{M_{l,i}} \cdot \mathbf{f_l}$. By performing sparse matrix multiplications, we reduce the complexity of the matrix-vector multiplication from $O(4^D)$ to $O(3^D)$.
    
    This step constructs $O(TLv_s)$ vectors, one for each matrix. Sparse matrix-vector and matrix-scalar multiplications run in time proportional to the number of non-zero entries. Therefore, this step has a complexity of $O(TL3^Dv_o^{avg})$.

    \item \textbf{Computing the final values}:
    Computing $P_c$ requires $O(nTL)$ time.
    The final Shapley/Banzhaf values or interaction values computed are via NumPy indexing over all $s$ vectors, each requiring $O(n)$ time. This yields an overall complexity of $O(nTLv_s)$. 
    
\end{enumerate}

By summing the complexities above, we obtain the total runtime of \algname for each case, as summarized in Table~\ref{full_complexity_table}.

\begin{table}[h!]
\resizebox{\columnwidth}{!}{%
\renewcommand{\arraystretch}{1.5}
\begin{tabular}{l|l}
Task & \algname Complexity \\\hline
Path Dependent on function $v$ & $O(nTLv_s + TL3^Dv_c^{avg})$ \\
Background on function $v$ & $O(mTL + nTLv_s + TL3^Dv_c^{avg})$ \\
Path Dependent SHAP & $O(nTLD + TL3^DD)$ \\
Path Dependent Banzhaf & $O(nTLD + TL3^DD)$ \\
Background SHAP & $O(mTL + nTLD + TL3^DD)$ \\
Background Banzhaf & $O(mTL + nTLD + TL3^DD)$ \\
Path Dependent SHAP IV &  $O(nTLD^2 + TL3^DD^2)$ \\
Path Dependent Banzhaf IV &  $O(nTLD^2 + TL3^DD^2)$ \\
Background SHAP IV & $O(mTL + nTLD^2 + TL3^DD^2)$ \\
Background Banzhaf IV & $O(mTL + nTLD^2 + TL3^DD^2)$ \\
\end{tabular}}
\caption{\label{full_complexity_table} \algname complexity. SHAP/Banzhaf IV refers to the task of calculating all Shapley/Banzhaf interaction values. Legend: $n$ = consumer data size, $m$ = background data size, $T$ = number of trees, $L$ = number of leaves per tree, $D$ = tree depth, $O(v_c^{avg})$ is function $v$ average complexity on all the cubes generated by a single root-to-leaf path and $O(v_s)$ is the number of unique feature subsets that $v$ returns from these cubes.}
\end{table}

\section{Expanded Experimental Section}
\label{sec:experimental_results_appendix}

This section expands on the experimental results presented in Sec.~\ref{sec:experimental_results} and provides additional details on the empirical correctness verification employed in our study.

\subsection{Additional Experimental Details}

\begin{enumerate}
    \item \textbf{Estimated Runtime}: the \shap Python package supports Baseline and Path-Dependent SHAP, whereas Background SHAP is limited to 100 background samples~\cite{shap_issue_background_100}, and Background Shapley interaction values are not supported~\cite{shap_issue_background_shap_iv}. Therefore, the running times of these Background SHAP are estimated. The running times of Shapley interaction values are also estimated due to RAM limitations. See the estimation methodology in Appendix.~\ref{sec:estimation}.
    \item \textbf{Modeling Framework}: Several state-of-the-art SHAP packages do not support XGBoost. To enable comparison on similar models, we used known alternatives: PLTreeShap~\cite{linear_background_shap} was run on a LightGBM model, and FastTreeSHAP~\cite{fast_tree_shap} on a scikit-learn RandomForest. All models consisted of 100 trees with a maximum depth of 6.
    \item \textbf{GPU SOTA}: The state-of-the-art (SOTA) GPU algorithm is GPUTreeSHAP. It is available in the \shap package, but requires cloning the repository and installing it locally~\cite{shap_gputree_doc}. The Path-Dependent variant of the same algorithm is also integrated into the XGBoost Python package and can be used without additional setup. Although both implementations are based on the same algorithm, the XGBoost version is significantly faster in practice. In our experiments, we used the XGBoost implementation as the SOTA for Path-Dependent SHAP and the \shap package as the SOTA for Background SHAP.
\end{enumerate}

See the results and performance comparison in Sect.~\ref{sec:experimental_results}.

\subsection{SOTA Running Time Estimation}
\label{sec:estimation}
Path-Dependent SHAP times represent actual measurements. We run the \shap Python package Background SHAP on both CPU and GPU using a background dataset of size 100 (the largest size possible~\cite{shap_issue_background_100}) and the full consumer dataset. Even with this small background dataset, the \shap took 3 minutes on IEEE-CIS and 90 minutes on KDD—already significantly slower than \algname, which processed the entire background dataset in just 12 seconds on IEEE-CIS and 162 seconds on KDD. To estimate \shap running time on the entire background dataset we scaled the sample running time by $0.01 \times Actual\_Background\_Size$. Additional experiments with background sizes of 10, 20, and 50 confirmed that \shap running time scales linearly with background size. 

Due to RAM limitations we estimate the running time of Shapley interaction values tasks using the same technic. See table~\ref{table_estimations} for estimation details.

\begin{table}[h!]

\resizebox{\columnwidth}{!}{%
\renewcommand{\arraystretch}{1.5}
\begin{tabular}{l|l|l|l|l|l}
Task & Framework & Dataset & $|B|$ & $|C|$ & Runtime \\\hline
BG SHAP & \shap & IEEE-CIS & 100 & all & 177 sec \\
BG SHAP & \shap & KDD & 100 & all & 5428 sec \\
BG SHAP & \shap (GPU) & IEEE-CIS &  100 & all & 11 sec \\
BG SHAP & \shap (GPU) & KDD & 100 & all & 160 sec \\
PD SHAP IV & \shap & IEEE-CIS & - & 100 & 101 sec \\
PD SHAP IV & \shap & KDD & - & 100 & 24 sec \\
PD SHAP IV & \shap (GPU) & IEEE-CIS & - & 10000 & 9 sec \\
PD SHAP IV & \shap (GPU) & KDD & - & 10000 & 0.8 sec \\
PD SHAP IV & FastTreeShap & KDD & - & 100000 & 445 sec \\
BG SHAP IV & PLTreeSHAP & IEEE-CIS & all & 10000 & 70 sec \\
BG SHAP IV & PLTreeSHAP & KDD & all & 10000 & 195 sec \\
\end{tabular}}
\caption{\label{table_estimations}Estimation approach for all estimated tasks. Legend: $|B|$: background data size, $|C|$ consumer data size, BG: Background, PD: Path Dependent (on Path Dependent SHAP $B$ is not used and thus $|B|=0$), \shap is the shap Python package. The "Runtime" column shows the time taken on the sampled data, before extrapolation. }
\end{table}

\subsection{Empirical Correctness Verification}

To validate the correctness of \algname, we compared the Shapley values and interaction values it computed with those generated by the \shap Python package. This comparison was carried out on the first \num{1000} consumers of the IEEE-CIS dataset, using the first \num{80} rows from the corresponding background dataset. The results showed strong agreement, with all values differing by at most 0.00001. Similar consistency was observed with Path-Dependent SHAP and Path-Dependent SHAP IV.

To validate the correctness of \algname's Banzhaf values, we compared its output on a small synthetic dataset to a direct exponential-time implementation. On 100 consumers and background size 3, all values differed by at most 0.00001.

\end{document}